%% file: transferability.tex
\documentclass[journal]{IEEEtran}

\usepackage[english]{babel}

\usepackage{ifpdf}

\usepackage{cite} 
\usepackage{url}
\usepackage{hyperref}

\ifCLASSINFOpdf
	\usepackage[pdftex]{graphicx}
	\graphicspath{{./figures/}}
\else
	\usepackage[dvips]{graphicx}
	\graphicspath{./figures/}
\fi
\usepackage{color}
\usepackage{pgf, tikz, pgfplots}
\usetikzlibrary{shapes, arrows, automata}
\usetikzlibrary{calc,hobby,decorations}

\usepackage[cmex10]{amsmath}
\usepackage{amsfonts, amssymb, amsthm}
\usepackage{mathrsfs}
\usepackage{mathbbol}

\usepackage{adjustbox}

\usepackage{array}
\usepackage{enumerate}
\usepackage{multirow}
\usepackage{rotating}
\usepackage{subcaption}
\usepackage{dsfont}

\input{my_sections.tex}

\input{mysymbol.sty}

\definecolor{penndarkestblue}{cmyk}{1,0.74,0,0.77}
\definecolor{penndarkerblue}{cmyk}{1,0.74,0,0.70}
\definecolor{pennblue}{cmyk}{0.99,0.66,0,0.57} 
\definecolor{pennlighterblue}{cmyk}{0.98,0.44,0,0.35}
\definecolor{pennlightestblue}{cmyk}{0.38,0.17,0,0.17} 

\definecolor{penndarkestred}{cmyk}{0,1,0.89,0.66}
\definecolor{penndarkerred}{cmyk}{0,1,0.88,0.55}
\definecolor{pennred}{cmyk}{0,1,0.83,0.42} 
\definecolor{pennlighterred}{cmyk}{0,1,0.6,0.24}
\definecolor{pennlightestred}{cmyk}{0,0.43,0.26,0.12} 

\definecolor{penndarkestgreen}{cmyk}{1,0,1,0.68}
\definecolor{penndarkergreen}{cmyk}{1,0,1,0.57}
\definecolor{penngreen}{cmyk}{1,0,1,0.44} 
\definecolor{pennlightergreen}{cmyk}{1,0,1,0.25}
\definecolor{pennlightestgreen}{cmyk}{0.43,0,0.43,0.13}

\definecolor{penndarkestorange}{cmyk}{0,0.65,1,0.49}
\definecolor{penndarkerorange}{cmyk}{0,0.65,1,0.33}
\definecolor{pennorange}{cmyk}{0,0.54,1,0.24} 
\definecolor{pennlighterorange}{cmyk}{0,0.32,1,0.13}
\definecolor{pennlightestorange}{cmyk}{0,0.15,0.46,0.06}
	
\definecolor{penndarkestpurple}{cmyk}{0,1,0.11,0.86}
\definecolor{penndarkerpurple}{cmyk}{0,1,0.13,0.82}
\definecolor{pennpurple}{cmyk}{0,1,0.11,0.71} 
\definecolor{pennlighterpurple}{cmyk}{0,1,0.05,0.46}
\definecolor{pennlightestpurple}{cmyk}{0,0.35,0.02,0.23}
	
\definecolor{pennyellow}{cmyk}{0,0.20,1,0.05} 
\definecolor{pennlightgray1}{cmyk}{0,0,0,0.05}
\definecolor{pennlightgray3}{cmyk}{0.01,0.01,0,0.18}
\definecolor{pennmediumgray1}{cmyk}{0.04,0.03,0,0.31}
\definecolor{pennmediumgray4}{cmyk}{0.08,0.06,0,0.54}
\definecolor{penndarkgray2}{cmyk}{0.09,0.07,0,0.71}
\definecolor{penndarkgray4}{cmyk}{0.1,0.1,0,0.92}

\def\Tr{\mathsf{T}}
\def\Hr{\mathsf{H}}

\newcommand{\vertiii}[1]{{\left\vert\kern-0.25ex\left\vert\kern-0.25ex\left\vert #1 \right\vert\kern-0.25ex\right\vert\kern-0.25ex\right\vert}}

\newtheorem{assumption}{\hspace{0pt}\bf AS\hspace{-0.15cm}}
\newtheorem{lemma}{\hspace{0pt}\bf Lemma}
\newtheorem{proposition}{\hspace{0pt}\bf Proposition}

\newtheorem{theorem}{\hspace{0pt}\bf Theorem}

\newtheorem{remark}{\hspace{0pt}\bf Remark}

\newtheorem{definition}{\hspace{0pt}\bf Definition}

\newtheorem{innercustomgeneric}{\customgenericname}
\providecommand{\customgenericname}{}
\newcommand{\newcustomtheorem}[2]{%
  \newenvironment{#1}[1]
  {%
   \renewcommand\customgenericname{#2}%
   \renewcommand\theinnercustomgeneric{##1}%
   \innercustomgeneric
  }
  {\endinnercustomgeneric}
}

\newcustomtheorem{customthm}{Theorem}
\newcustomtheorem{customlemma}{Lemma}

\begin{document}

\title{Transferability Properties of Graph Neural Networks}

\author{Luana~Ruiz, Luiz~F.~O.~Chamon~
        and~Alejandro~Ribeiro
\thanks{This work in this paper was supported by NSF TRIPODS, NSF CCF 1717120, ARO W911NF1710438, ARL DCIST CRA W911NF-17-2-0181, and Theorinet Simons. L. Ruiz and A. Ribeiro are with the Dept. of Electrical and Systems Eng., Univ. of Pennsylvania. L. F. O. Chamon is with the Simons Institute, U. C. Berkeley.  
}
}

\markboth{IEEE TRANSACTIONS ON SIGNAL PROCESSING (ACCEPTED)}%
{Transferability}

\maketitle

\begin{abstract}

Graph neural networks (GNNs) are composed of layers consisting of graph convolutions and pointwise nonlinearities. Due to their invariance and stability properties, GNNs are provably successful at learning representations from data supported on moderate-scale graphs. However, they are difficult to learn on large-scale graphs. In this paper, we study the problem of training GNNs on graphs of moderate size and transferring them to large-scale graphs. We use graph limits called graphons to define limit objects for graph filters and GNNs---graphon filters and graphon neural networks ($\bbW$NNs)---which we interpret as generative models for graph filters and GNNs. We then show that graphon filters and $\bbW$NNs can be approximated by graph filters and GNNs sampled from them on weighted and stochastic graphs. Because the error of these approximations can be upper bounded, by a triangle inequality argument we can further bound the error of transferring a graph filter or a GNN across graphs. 
Our results show that (i) the transference error decreases with the graph size, and (ii) that graph filters have a transferability-discriminability tradeoff that in GNNs is alleviated by the scattering behavior of the nonlinearity. These findings are demonstrated empirically in a recommendation problem and in a decentralized control task.

\end{abstract}


\IEEEpeerreviewmaketitle


\section{Introduction} \label{sec:intro}

\input{intro-transferability.tex}


\section{Preliminary definitions} \label{sec:prelim}

\input{prelim-transferability.tex}


\section{Transferability of graph filters} \label{sec:filter-transferability}

\input{filter-transferability.tex}

\section{Transferability of graph neural networks} \label{sec:gnn-transferability}

\input{gnn-transferability.tex}


\section{Numerical Experiments} \label{sec:sims}

\input{sims-transferability.tex}

\section{Conclusions} \label{sec:conclusions}

\input{conclusions-transferability.tex}

\bibliographystyle{IEEEtran}
\bibliography{myIEEEabrv,bib-transferability}

\appendices

\section{Proof of Thm. 1} \label{sec:appendixA}

To prove Thm. \ref{lemma:graphon-graph-filter}, we will need the following propositions.

\begin{proposition} \label{prop:eigenvalue_diff}
Let $\bbW:[0,1]^2\to[0,1]$ and $\bbW':[0,1]^2\to[0,1]$ be two graphons with eigenvalues given by $\{\lambda_i(T_\bbW)\}_{i\in\mbZ\setminus\{0\}}$ and $\{\lambda_i(T_{\bbW'})\}_{i\in\mbZ\setminus\{0\}}$, ordered according to their sign and in decreasing order of absolute value. Then, for all $i \in \mbZ \setminus \{0\}$, the following inequalities hold
\begin{equation*}
|\lambda_i(T_{\bbW'})-\lambda_i(T_\bbW)| \leq \vertiii{T_{\bbW'-\bbW}} \leq \|\bbW'-\bbW\|\ .
\end{equation*}
\end{proposition}
\begin{proof} \renewcommand{\qedsymbol}{}
See \cite[Prop. 4]{ruiz20-transf}.
\end{proof}

\begin{proposition}\label{thm:davis_kahan}
Let $T$ and $T^\prime$ be two self-adjoint operators on a separable Hilbert space $\ccalH$ whose spectra are partitioned as $\gamma \cup \Gamma$ and $\omega \cup \Omega$ respectively, with $\gamma \cap \Gamma = \emptyset$ and $\omega \cap \Omega = \emptyset$. If there exists $d > 0$ such that $\min_{x \in \gamma,\, y \in \Omega} |{x - y}| \geq d$ and $\min_{x \in \omega,\, y \in \Gamma}|{x - y}| \geq d$, then the spectral projections $E_T(\gamma)$ and $E_{T^\prime}(\omega)$ satisfy
\begin{equation*}\label{eqn:davis_kahan}
	\vertiii{E_T(\gamma) - E_{T^\prime}(\omega)} \leq \frac{\pi}{2} \frac{\vertiii{{T - T^\prime}}}{d}
\end{equation*}
\end{proposition}
\begin{proof} \renewcommand{\qedsymbol}{}
See \cite{seelmann2014notes}.
\end{proof}

To prove Lemma \ref{lemma:graphon-graph-filter}, we fix a constant $0 < c < 1$ and decompose the filter $h$ into filters $h^{\geq c}$ and $h^{< c}$ defined as
\begin{align}
&h^{\geq c}(\lambda) \begin{cases} 0 &\mbox{if } |\lambda| < c \\
h(\lambda)-h(c) & \mbox{if } \lambda \geq c \\
h(\lambda)-h(-c) & \mbox{if } \lambda \leq -c \end{cases} \quad \mbox{and}\\
&h^{<c}(\lambda) \begin{cases} h(\lambda) &\mbox{if } |\lambda| < c \\
h(c) & \mbox{if } \lambda \geq c \\
h(-c) & \mbox{if } \lambda \leq -c \end{cases}
\end{align}
so that $h= h^{\geq c} + h^{<c}$. Note that both functions have the same Lipschitz constants as $h$. We start by analyzing the transferability of $h^{\geq c}$.

Let $T_{\bbH}^{\geq c}$ and $T_{\bbH_n}^{\geq c}$ be graphon filters with filter function $h^{\geq c}$ on the graphons $\bbW$ and $\bbW_n$ respectively. Using the triangle inequality, we can write the norm difference $\|T_{\bbH}^{\geq c}X-T_{\bbH_n}^{\geq c}X_n\|$ as
\begin{align*}
\begin{split}
\left\|T^{\geq c}_\bbH X-T^{\geq c}_{\bbH_n}X_n\right\| &= \left\|T^{\geq c}_\bbH X +T^{\geq c}_{\bbH_n} X - T^{\geq c}_{\bbH_n}X -T^{\geq c}_{\bbH_n}X_n\right\| \\
&\leq \left\|T^{\geq c}_\bbH X-T^{\geq c}_{\bbH_n} X\right\| \hspace{1mm} \mbox{\textbf{(1)}} \\
&+ \left\|T^{\geq c}_{\bbH_n}\left(X-X_n\right)\right\| \hspace{2.6mm} \mbox{\textbf{(2)}}
\end{split}
\end{align*}
where the LHS is split between terms \textbf{(1)} and \textbf{(2)}. 

Writing the inner products $\int_0^1 X(u) \varphi_i(u) du$ and $\int_0^1 X(u) \varphi^n_i(u) du$ as $\hat{X}(\lambda_i)$ and $\hat{X}(\lambda^n_i)$ for simplicity, we can then express \textbf{(1)} as
\begin{align*}
\begin{split}
\big\|T^{\geq c}_\bbH &X-T^{\geq c}_{\bbH_n} X\big\| \\
&=\Bigg\|\sum_{i} h^{\geq c}(\lambda_i)\hat{X}(\lambda_i) \varphi_i - \sum_i h^{\geq c}(\lambda_i^n)\hat{X}(\lambda^n_i)\varphi_i^n\Bigg\| \\
&=\left\|\sum_{i} h^{\geq c}(\lambda_i)\hat{X}(\lambda_i) \varphi_i -  h^{\geq c}(\lambda_i^n)\hat{X}(\lambda^n_i)\varphi_i^n\right\| . 
\end{split}
\end{align*}
Using the triangle inequality, this becomes
\begin{align*}
\begin{split}
\Big\|T^{\geq c}_\bbH &X-T^{\geq c}_{\bbH_n} X\Big\| \\
&= \Bigg\|\sum_{i} h^{\geq c}(\lambda_i)\hat{X}(\lambda_i) \varphi_i + h^{\geq c}(\lambda_i^n)\hat{X}(\lambda_i)\varphi_i \\
&\quad - h^{\geq c}(\lambda_i^n)\hat{X}(\lambda_i)\varphi_i - h^{\geq c}(\lambda_i^n)\hat{X}(\lambda_i^n)\varphi_i^n \Bigg\| \\
&\leq \left\|\sum_{i} \left(h^{\geq c}(\lambda_i)-h^{\geq c}(\lambda_i^n)\right)\hat{X}(\lambda_i) \varphi_i \right\| \hspace{3.8mm} \mbox{  \textbf{(1.1)}} \\
&+ \left\|\sum_{i} h^{\geq c}(\lambda_i^n) \left(\hat{X}(\lambda_i) \varphi_i - \hat{X}(\lambda^n_i)\varphi_i^n \right) \right\| \hspace{1mm} \mbox{  \textbf{(1.2)}}
\end{split}
\end{align*}
where we have now split \textbf{(1)} between \textbf{(1.1)} and \textbf{(1.2)}. 

Focusing on \textbf{(1.1)}, note that the filter's Lipschitz property allows writing $|h(\lambda_i)-h^{\geq c}(\lambda_i^n)| \leq A_h |\lambda_i-\lambda_i^n|$. 
Hence, using Prop. \ref{prop:eigenvalue_diff} together with the Cauchy-Schwarz inequality, we get
\begin{align} \label{eqn:proofpart1}
\begin{split}
\Bigg\|\sum_{i} \Big(h^{\geq c}(\lambda_i)&-h^{\geq c}(\lambda_i^n)\Big)\hat{X}(\lambda_i) \varphi_i \Bigg\| \\
&\leq A_h \|\bbW-\bbW_n\|\left\|\sum_{i} \hat{X}(\lambda_i) \varphi_i\right\| \\
&= A_h \|\bbW-\bbW_n\|\|X\|\ .
\end{split}
\end{align}

For \textbf{(1.2)}, we use the triangle and Cauchy-Schwarz inequalities to write
\begin{align*}
\begin{split}
\Bigg\|\sum_{i} h^{\geq c}(\lambda_i^n) &\left(\hat{X}(\lambda_i) \varphi_i - \hat{X}(\lambda^n_i)\varphi_i^n \right) \Bigg\| \\
&= \Bigg\|\sum_{i} h^{\geq c}(\lambda_i^n) \bigg(\hat{X}(\lambda_i) \varphi_i + \hat{X}(\lambda_i) \varphi_i^n \\
&\quad \quad \quad \quad \quad - \hat{X}(\lambda_i) \varphi_i^n - \hat{X}(\lambda_i^n)\varphi_i^n \bigg) \Bigg\| \\
&\leq \left\|\sum_{i} h^{\geq c}(\lambda_i^n) \hat{X}(\lambda_i) (\varphi_i - \varphi_i^n) \right\| \\
&\quad \quad \quad + \left\|\sum_{i} h^{\geq c}(\lambda_i^n) \varphi_i^n \langle X, \varphi_i - \varphi_i^n \rangle \right\| \\
&\leq 2\sum_{i} \|h^{\geq c}(\lambda_i^n)\| \|X\| \|\varphi_i - \varphi_i^n\|\ .
\end{split}
\end{align*}
Using Prop. \ref{thm:davis_kahan} with $\gamma = \lambda_i$ and $\omega = \lambda_i^n$, we then get 
\begin{align*}
\begin{split}
\Bigg\|\sum_{i} h^{\geq c}(\lambda_i^n) &\big(\hat{X}(\lambda_i) \varphi_i - \hat{X}(\lambda^n_i)\varphi_i^n \big) \Bigg\| \\
&\leq \|X\| \sum_{i} \|h^{\geq c}(\lambda_i^n)\| \frac{\pi\vertiii{T_\bbW - T_{\bbW_n}}}{d_i}
\end{split}
\end{align*}
where $d_i$ is the minimum between $\min(|\lambda_i - \lambda_{i+1}^n|,|\lambda_i-\lambda_{i-1}^n|)$ and $\min(|\lambda^n_i - \lambda_{i+1}|,|\lambda^n_i-\lambda_{i-1}|)$ for each $i$. Since $\delta^c_{\bbW\bbW_n} \leq d_i$ for all $i$ and $\vertiii{T_\bbW - T_{\bbW_n}} \leq \|\bbW-\bbW_n\|$ (i.e., the Hilbert-Schmidt norm dominates the operator norm), this becomes
\begin{align*}
\begin{split}
\Bigg\|\sum_{i} h^{\geq c}(\lambda_i^n) &\left(\hat{X}(\lambda_i) \varphi_i - \hat{X}(\lambda^n_i)\varphi_i^n \right) \Bigg\| \\
&\leq \frac{\pi\|\bbW - {\bbW_n}\|}{\delta^c_{\bbW\bbW_n}} \|X\| \sum_{i} \|h^{\geq c}(\lambda_i^n)\|\ .
\end{split}
\end{align*}
The final bound for \textbf{(1.2)} is obtained by noting that $|h^{\geq c}(\lambda)|<1$ and $h^{\geq c}(\lambda) = 0$ for $|\lambda| < c$. Since there are a total of $B^c_{\bbW_n}$ eigenvalues $\lambda_i^n$ for which $|\lambda_i^n| \geq c$, we get
\begin{align} \label{eqn:proofpart2}
\begin{split}
\left\|\sum_{i} h^{\geq c}(\lambda_i^n) \left(\hat{X}(\lambda_i) \varphi_i - \hat{X}(\lambda^n_i)\varphi_i^n \right) \right\| \\
\leq \frac{\pi\|\bbW - {\bbW_n}\|}{\delta^c_{\bbW\bbW_n}} \|X\| B^c_{\bbW_n}\ .
\end{split}
\end{align}

The bound for \textbf{(2)} follows immediately from the Cauchy-Schwarz inequality. Since $|h(\lambda)|<1$, the norm of the operator $T_{\bbH_n}$ is bounded by 1. Hence,
\begin{equation} \label{eqn:proofpart3}
\|T^{\geq c}_{\bbH_n}(X-X_n)\| \leq \|X-X_n\|
\end{equation}
which completes the bound on $\|T^{\geq c}_\bbH X-T^{\geq c}_{\bbH_n}X_n\|$.

Next we analyze the transferability of the filter $h^{<c}$. Note that, because this filter only varies in the interval $(-c,c)$ and has Lipschitz constant $a_h$, we can bound $\|T^{<c}_{\bbH}X-T^{<c}_{\bbH_n}X_n)\|$ as
\begin{align*}
\begin{split}
\|T^{<c}_{\bbH}X-T^{<c}_{\bbH_n}X_n)\| \leq \|(h(0)+a_h c)X
-(h(0)-a_h c)X_n)\|  \\
\leq |h(0)|\|X-X_n\|
+ a_h c\|X+X_n\|
\end{split}
\end{align*}
where the second inequality follows from the triangle and Cauchy-Schwarz inequalities. Using the fact that $|h(\lambda)| < 1$, and adding and subtracting $X$ to $\|X+X_n\|$ to apply the triangle inequality once again, his bound becomes
\begin{align} \label{eqn:proofpart4}
\begin{split}
\|T^{<c}_{\bbH}X-T^{<c}_{\bbH_n}X_n)\|
\leq \|X-X_n\|
+ a_h c\|2X +X_n - X\| \\
\leq (1+a_h c)\|X-X_n\| + 2a_h c \|X\|\ .
\end{split}
\end{align}

Finally, to obtain the transferability bound for $h$ observe that, since $h = h^{\geq c} + h^{<c}$, we can write
\begin{align*}
\begin{split}
\|T_{\bbH}X-T_{\bbH_n}X_n)\| = \|T^{\geq c}_{\bbH}X+T^{< c}_{\bbH}X-T^{\geq c}_{\bbH_n}X_n-T^{< c}_{\bbH_n}X_n)\| \\
\leq \|T^{\geq c}_{\bbH}X-T^{\geq c}_{\bbH_n}X_n\| + \|T^{< c}_{\bbH}X-T^{< c}_{\bbH_n}X_n\|
\end{split}
\end{align*}
i.e., the error made when transferring $h$ can be bounded by the sum of the transferability bounds of $h^{\geq c}$ and $h^{<c}$. Putting together equations \eqref{eqn:proofpart1}, \eqref{eqn:proofpart2}, \eqref{eqn:proofpart3} and \eqref{eqn:proofpart4} thus concludes the proof.

\pagebreak

\input{app-transferability.tex}

\end{document}

%% file: my_sections.tex
\usepackage{needspace}

\newcommand{\myparagraph}[1]{\needspace{1\baselineskip}\medskip\noindent {\bf #1}}



%% file: intro-transferability.tex


Graph neural networks (GNNs) are deep learning architectures for network data popularized by their state-of-the-art performance in a number of learning tasks \cite{kipf17-classifgcnn,defferrard17-cnngraphs,gama18-gnnarchit,bronstein17-geomdeep}. 
Besides working well in practice, GNNs are verifiably invariant to node relabelings and stable to graph perturbations \cite{ruiz19-inv,gama2019stability}, which are properties that help understand their practical success. Each layer of a GNN composes one or more graph convolutions and a pointwise nonlinearity; the permutation equivariance and the stability properties are inherited from the graph convolution. Because it is a local graph operator, the graph convolution is also to credit for the ability to implement GNNs in a distributed way \cite{gama2020graphs}.

An important consequence of locality is that, unlike in a fully connected neural network (FCNN), in a GNN the number of parameters does not depend on the number of nodes. In fact, the number of parameters of a GNN is typically much smaller than the graph size. In theory, this should mean that GNNs can scale to graphs with large number of nodes. In practice, learning GNNs for large-scale graphs is difficult. One reason for that is that graph convolutions need full knowledge of the graph structure, which may be hard to measure for large-scale graphs. Another is that they require calculating large matrix-vector multiplications, which, for an $n$-node graph, have in the order of $\ccalO(n^2)$ complexity if the graph is not sparse. 

Although the independence between the GNN weights and the graph does not guarantee scalability, it allows for the same GNN to be applied to different graphs.
Thus, we could train the GNN on a graph of moderate size and transfer the learned weights to a larger graph. For this to work, the graph on which the GNN is trained---the training graph---and the one on which it is executed---the target graph---have to be similar. 
In this paper, we define similar graphs as graphs that are associated with the same \textit{graphon} \cite{lovasz2012large,lovasz2006limits,borgs2008convergent,borgs2012convergent}. Graphons can be used to describe families of similar graphs because they are both the limits of convergent graph sequences and generative models for weighted (Definitions \ref{def:template} and \ref{def:weighted}) and stochastic graphs (Definition \ref{def:stochastic}). They proved to be a useful model in large-scale graph problems in statistics \cite{wolfe2013nonparametric,xu2017rates,gao2015rate}, game theory \cite{parise2019graphon}, network science \cite{avella2018centrality,vizuete2020laplacian} and controls \cite{gao2018graphon}.

To understand whether a GNN learned on the training graph would work well on the target graph, we need to estimate the difference between the outputs of the GNN on these graphs or its \emph{transference error}. 
To do so, we define the graphon filter and the graphon neural network ($\bbW$NN), which are the limit objects of the graph filter and the GNN on the graphon \cite{ruiz20-transf}. By interpreting the graphon filter and the $\bbW$NN as generative models for graph filters and GNNs, we show that they can be approximated by graph filters (Thm. \ref{lemma:graphon-graph-filter}) and GNNs (Thm. \ref{thm:graphon-graph-nn}) sampled from them. To illustrate these results with concrete examples, in Props. \ref{lemma:graphon-graph-filter-template}--\ref{thm:graphon-graph-filter-stochastic}, \ref{thm:graphon-graph-nn-stochastic} we particularize these theorems to template, weighted and stochastic graphs.
Since the error of these approximations can be upper bounded, by a triangle inequality argument we can further bound the error of transferring graph filters (Prop. \ref{thm:graph-graph-filter-stochastic}) and GNNs (Prop. \ref{thm:graph-graph-nn-stochastic}) across graphs. 

These results have two important implications. The first is that both graph filters and GNNs have an inherent \emph{transferability property}, i.e., they can be transferred across similar graphs with performance guarantees. But although transferability increases with the size of both the training and the target graph, in the finite-sample regime (i.e., on graphs of finite size) some spectral components are not transferable. 
The second is that there is a tradeoff between the transferability and the spectral discriminability of the graph convolution. While it would be natural to expect this tradeoff to be inherited by GNNs, numerical experiments on a movie recommendation problem and a decentralized control task (Sec. \ref{sec:sims}) show that GNNs are more transferable than convolutional filters. A plausible explanation is that, in GNNs, the transferability-discriminability tradeoff is alleviated by the scattering behavior of the nonlinearities; see Sec. \ref{sbs:discussion}.


\subsection{Related Work}

This paper follows from a series of works on graphon signal processing (WSP), including \cite{ruiz2020graphonsp}---which introduces WSP and proves convergence of graph to graphon filters---and \cite{ruiz20-transf}---which defines $\bbW$NNs and studies GNN transferability over a narrow class of weighted graphs, where nodes are sampled from a regular partition and edges have deterministic weights given by the graphon. Herein, we extend upon the results of \cite{ruiz20-transf} by considering graphs with both stochastic nodes and edges, and proving a tighter transferability bound (of order $\ccalO(n^{-1})$) than the $\ccalO(n^{-1/2})$ bound in \cite{ruiz20-transf}.

Transferability of GNNs has also been studied in \cite{levie2019transferability,maskey2021transferability,keriven2020convergence}. These papers focus on the same problem we address in this paper, but use different methods and arrive at different results. More importantly, what sets our work apart is that it unveils a tradeoff between the transferability and the spectral behavior of the graph convolution (and thus of GNNs)---which is not brought up in any other works.

Paper \cite{levie2019transferability} considers graphs sampled from generic topological spaces as opposed to graphons. Graphs and graph signals are obtained by application of a generic sampling operator, and mapped back to the topological space via an analogous reconstruction map. The bound on the GNN transferability error depends on the quality of the sampling-reconstruction process, whose error is assumed bounded. Moreover, the transferability analysis in \cite{levie2019transferability} is restricted to signals in Paley-Wiener spaces, i.e., signals which are bandlimited. In contrast, in our paper we give three examples of how one might sample graphs and graph signals from graphons and graphon signals (Definitions \ref{def:template}--\ref{def:stochastic}), and do not make any bandlimitedness assumptions on the input signal. The latter is particularly important because it allows understanding the transferability behavior of the GNN when the data has high-frequency information.

Same as in our paper, paper \cite{maskey2021transferability} considers graphs sampled from graphons and proves transferability of GNNs by  showing that graph convolutions are transferable. However, while in our paper we do not impose any restrictions on the input signal and only require the filter $h$ to be Lipschitz, in \cite{maskey2021transferability} more restrictive assumptions are placed on either the signal---which has to be in $L^p$ for $p>2$---or the filter $h$---where $h'$ must be Lipschitz and $h(0)=0$. Our assumptions lead to a detailed analysis of the spectral implications of transferability, showing that graph convolutions have a transferability-discriminability tradeoff and suggesting a hypothesis for why this tradeoff is less pronounced in GNNs. These considerations, which are the main result from our analysis, are not made in \cite{maskey2021transferability}.

Paper \cite{keriven2020convergence} gives finite-sample bounds for the convergence and stability of GNNs on graphs sampled from a random graph models $G(n,p_n)$ where the probability $p_n$ of sampling an edge is given by a weighting of a graphon and a decreasing function of $n$. As a consequence, their analysis applies to graphs ranging from dense (i.e, sampled from graphons) to moderately sparse. The main difference of our paper with respect to \cite{keriven2020convergence} is that their transferability bound depends on the number of coefficients of the graph convolution, whereas ours is a function of the filter's Lipschitz constant. It is by focusing on the Lipschitz property of the filter---without imposing any restrictions on its number of coefficients---that we are able to visualize the relationship between the transferability and the spectral discriminability of the graph convolution. This relationship is not discussed in \cite{keriven2020convergence}.

Somewhat related to our paper, \cite{thangtransferability} analyzes transferability from a different angle by assuming that the graph structure is a latent function of the node features, which can be seen as a homophily assumption. This latent function then defines a graph family in which GNNs are shown to transfer well empirically.
Finally, GNN stability is closely related to transferability and is studied in the seminal paper by \cite{gama2019stability}---which follows from stability analyses of graph scattering transforms \cite{gama2019stabilityscat}---, as well as in \cite{elvin19-spectral,kenlay2021interpretable,kenlay2021stability}. Other related work includes the convergence analyses of graph filters in \cite{ruiz2021graphon-filters} and \cite{morency2021graphon}. These differ from our paper in that they only focus on graph filters, and only give asymptotic transferability results. 


\myparagraph{Notation.} All norms $\|\cdot\|$ refer to the $L^2$ norm. We use $\vertiii{\cdot}$ to denote the operator norm induced by the $L^2$ norm.

%% file: prelim-transferability.tex


We consider data supported on undirected graphs $\bbG=(\ccalV,\ccalE,\ccalW)$ where $\ccalV$ is a set of $|\ccalV|=n$ nodes, $\ccalE \subseteq \ccalV \times \ccalV$ is a set of edges, and $\ccalW: \ccalE \to \reals$ is a symmetric function assigning real weights to the edges in $\ccalE$. This data is represented in the form of vectors $\bbx \in \reals^n$ called graph signals where $[\bbx]_i$ is the value of the signal at node $i$ \cite{shuman13-mag,ortega2018graph}. To parametrize signal processing operations that depend on the graph topology by the graph $\bbG$, a generic graph matrix representation $\bbS \in \reals^{n \times n}$, called the graph shift operator (GSO), is defined. The GSO $\bbS$ is such that $[\bbS]_{ij} = s_{ij} \neq 0$ if and only if $i=j$ or $(i,j) \in \ccalE$.
The graph adjacency $[\bbS]_{ij} = [\bbA]_{ij} = \ccalW(i,j)$ \cite{sandryhaila13-dspg}, the graph Laplacian $\bbS = \bbL=\diag (\bbA\boldsymbol{1})-\bbA$ \cite{shuman13-mag}, and the random walk Laplacian $\bbS = \diag (\bbA \boldsymbol{1})^{-1} \bbL$ \cite{boukrab2021random} are all examples of matrices that satisfy this property. Unless otherwise specified, in this paper we consider $\bbS=\bbA$ and replace $\ccalW$ by $\bbS$ in the triplet $\bbG=(\ccalV,\ccalE,\bbS)$ to simplify notation. 

The GSO gets its name from the fact that it defines a notion of shift for signals $\bbx$ on the graph $\bbG$. This is because, at each node $i$, the signal $\bbz = \bbS\bbx$ resulting from an application of the GSO to $\bbx$ can be written elementwise as
\begin{equation*}
[\bbz]_i = \sum_{j=1}^n [\bbS]_{ij}[\bbx]_j = \sum_{j:[\bbS]_{ij} \neq 0} [\bbS]_{ij}[\bbx]_j = \sum_{j \in \ccalN(i)} [\bbS]_{ij}[\bbx]_j
\end{equation*}
where $\ccalN(i) = \{j: [\bbS]_{ij} \neq 0\}$ is the neighborhood of $i$. In other words, the outcome of $\bbS\bbx$ at node $i$ is a result of the neighboring nodes $j$ \textit{shifting} their data values $[\bbx]_j$, each weighted by the proximity measure $[\bbS]_{ij}$, to the node $i$. 

The notion of a graph shift, in turn, allows defining linear shift-invariant (LSI) or \textit{convolutional} filters on graphs. Let $\bbh = [h_0, \ldots, h_{K-1}]^\Tr$ be a vector of $K$ real coefficients. The graph convolution with coefficients $\bbh$ is defined as \cite{du2018graph,segarra17-linear}
\begin{equation} \label{eqn:graph_convolution}
\bbH(\bbS) \bbx = \sum_{k=0}^{K-1} h_k \bbS^k \bbx 
\end{equation}
i.e., it is given by the application of an order $K$ polynomial of $\bbS$ to the signal $\bbx$, similarly to how the convolution is defined as a shift-and-sum operation in Euclidean domains.

Because $\bbS$ is symmetric, it can be diagonalized as $\bbS = \bbV \bbLam \bbV^{\Hr}$. The matrix $\bbLam$ is a diagonal matrix whose diagonal elements are the graph eigenvalues which we assume to be ordered according to their sign and in decreasing order of absolute value, i.e., $ \lambda_{-1} \leq \lambda_{-2} \leq \ldots \leq 0 \leq \ldots \leq \lambda_2 \leq \lambda_1$. These are interpreted as \textit{graph frequencies}.  We consider \emph{high magnitude} eigenvalues to be low frequencies and \emph{low magnitude} eigenvalues to be high frequencies because we relate the graph signal's frequency with the absolute value of its normalized Rayleigh quotient. Explicitly, we define the frequency of a graph signal $\bbx$ as $1-\smash{\left|\frac{\bbx^\Hr\bbA\bbx}{n\bbx^\Hr\bbx}\right|}$. The graph eigenvectors, given by the columns of $\bbV$, can be seen as the graph's \textit{oscillation modes} \cite{sandryhaila14-freq}. Importantly, the eigenvectors in $\bbV$ form an orthonormal basis of $\reals^{n \times n}$ which we call the graph spectral basis. The graph Fourier transform (GFT) of $\bbx$ is defined as the projection of $\bbx$ onto this basis,
\begin{equation} \label{eqn:gft}
\hbx = \mbox{GFT}\{\bbx\} = \bbV^\Hr \bbx.
\end{equation}
Similarly, the inverse GFT is defined as $\bbx = \bbV \hbx$.

Substituting $\bbS=\bbV\bbLam\bbV^\Hr$ into \eqref{eqn:graph_convolution} and calculating the GFT of $\bby = \bbH(\bbS)\bbx$, we get
\begin{equation} \label{eqn:spec-lsi-gf}
\hby = \bbV^\Hr \bbH(\bbS)\bbx = \sum_{k=0}^{K-1} h_k \bbLam^k \bbV^\Hr \bbx = \sum_{k=0}^{K-1} h_k \bbLam^k \hbx
\end{equation}
from which we obtain the spectral representation of the graph convolution $h(\lambda)= \sum_{k=0}^{K-1} h_k \lambda^k$. 
There are two important things to note about the function $h$. First, while the GFTs of $\bbx$ and $\bby$ depend on the eigenvectors of $\bbS$, the spectral response of $\bbH(\bbS)$ is obtained by evaluating $h$ \textit{only} at the eigenvalues of the graph. This is illustrated in Fig. \ref{fig:eigenvalues_and_filter}, where the red curve represents $h(\lambda)$ and the green lines represent the eigenvalues of the graph $\bbG$. Second, on a given graph $\bbG$ polynomial filters $\bbH(\bbS)$ of order $K=N$ can be used to implement filters with spectral response given by any arbitrary function $f$ at the eigenvalues $\lambda_i$, i.e., $h(\lambda_i)=f(\lambda_i)$, as a consequence of the Cayley-Hamilton theorem.
In particular, if $f$ is analytic, this function can be approximated by a polynomial filter $h(\lambda)$ for all $\lambda$. The error of the approximation will depend on the filter order $K$.


\subsection{Graph Neural Networks} \label{sbs:gnns}

A GNN is a deep convolutional architecture where each layer consists of (i) a bank of convolutional filters like the one in \eqref{eqn:graph_convolution} and (ii) a nonlinear activation function.
The convolutional filterbank transforms the $F_{\ell-1}$ features $\bbx_{\ell-1}^g$ from layer $\ell-1$ into $F_\ell$ intermediate linear features given by
\begin{equation} \label{eqn:gcn_layer1}
\bbu^f_{\ell} =  \sum_{g=1}^{F_{\ell-1}}\bbH_{\ell}^{fg} (\bbS) \bbx_{\ell-1}^g 
\end{equation}
where $1 \leq f \leq F_{\ell}$. Specifically, the filter $\bbH_\ell^{fg}(\bbS)$ maps feature $g$ of layer $\ell-1$ into feature $f$ of layer $\ell$. The collection of filters $\bbH_\ell^{fg}(\bbS)$ for all $f,g$ yields a filterbank $\{\bbH_\ell^{fg}(\bbS)\}_{f,g}$ with a total of $F_{\ell-1}\times F_\ell$ filters like the one in \eqref{eqn:graph_convolution}. 

The activation function is usually a pointwise nonlinearity such as the ReLU or sigmoid, but localized implementations incorporating the graph structure have also been proposed \cite{ruiz19-inv}. 
Letting $\sigma$ denote the activation function, the $\ell$th layer of the GNN can be written as
\begin{equation} \label{eqn:gcn_layer2}
\bbx^f_{\ell} = \sigma \left(\bbu^f_\ell \right)\text{.}
\end{equation}
Because $\sigma$ is pointwise, at each node $i \in \ccalV$ we have $[\bbx_\ell^f]_i = \sigma([\bbu_\ell^f]_i)$.
If the total number of layers of the GNN is $L$, \eqref{eqn:gcn_layer1}--\eqref{eqn:gcn_layer2} are repeated for layers $\ell = 1$ through $\ell = L$, the output of each layer being the input to the next. At $\ell=1$, the input features $\bbx^g_0$ are the input data $\bbx^g$ for $1 \leq g \leq F_0$. The GNN output is given by the features of the last layer, i.e., $\bby^f = \bbx_L^f$. 

The goal of graph machine learning is to obtain meaningful representations $\bby$ from input data $\bbx$ (e.g., through empirical risk minimization). This can be thought of as learning a map
$\Phi: \bbx \mapsto \bby$ where $\Phi$ is parametric on the graph. The function $\Phi$  could be, for instance, parametrized by a graph filter $\bby = \bbH(\bbS)\bbx = \sum_{k=0}^{K-1}h_k \bbS^k \bbx$ [cf. \eqref{eqn:graph_convolution}]. In this case, we would write
\begin{equation} \label{eqn:filter_map}
\bby = \Phi(\bbx; \bbh, \bbS)
\end{equation}
where $\bbh = [h_0, \ldots, h_{K-1}]$. Alternatively, $\Phi$ could be parametrized by a $L$-layer GNN with inputs $\bbx = \{\bbx^g\}_{g=1}^{F_0}$ and outputs $\bby = \{\bby^f\}_{f=1}^{F_L}$ [cf. \eqref{eqn:gcn_layer1}--\eqref{eqn:gcn_layer2}]. In this case, we would write
\begin{equation} \label{eqn:gcn_map}
\bby = \Phi(\bbx; \ccalH, \bbS)
\end{equation}
where $\ccalH = [\bbh^{fg}_\ell]_{\ell f g}$ is a tensor grouping the GNN parameters $\bbh^{fg}_\ell$ for all layers $1 \leq \ell \leq L$ and features $1 \leq f \leq F_\ell$ and $1 \leq g \leq F_{\ell-1}$. This notation is particularly convenient for GNNs because it summarizes \eqref{eqn:gcn_layer1}--\eqref{eqn:gcn_layer2} for $\ell=1,\ldots,L$ into a single function $\Phi$. It also highlights the fact that the graph filter parameters $\bbh$ and the GNN parameters $\ccalH$ are independent of the GSO $\bbS$ (and thus of the graph $\bbG$).


\subsection{Graphs and Graphons} \label{sbs:graphons}

Graphons are the limit objects of sequences of dense undirected graphs. They are defined as functions $\bbW: [0,1]^2 \to [0,1]$ where $\bbW$ is bounded, measurable and symmetric. A sequence of graphs converging to the constant or Erd\"os-R\'enyi graphon is shown in Fig. \ref{fig:graphon_example}.

\begin{figure*}[t]
\centering
\begin{subfigure}{0.22\textwidth}
\centering
\includegraphics[width=3.7cm]{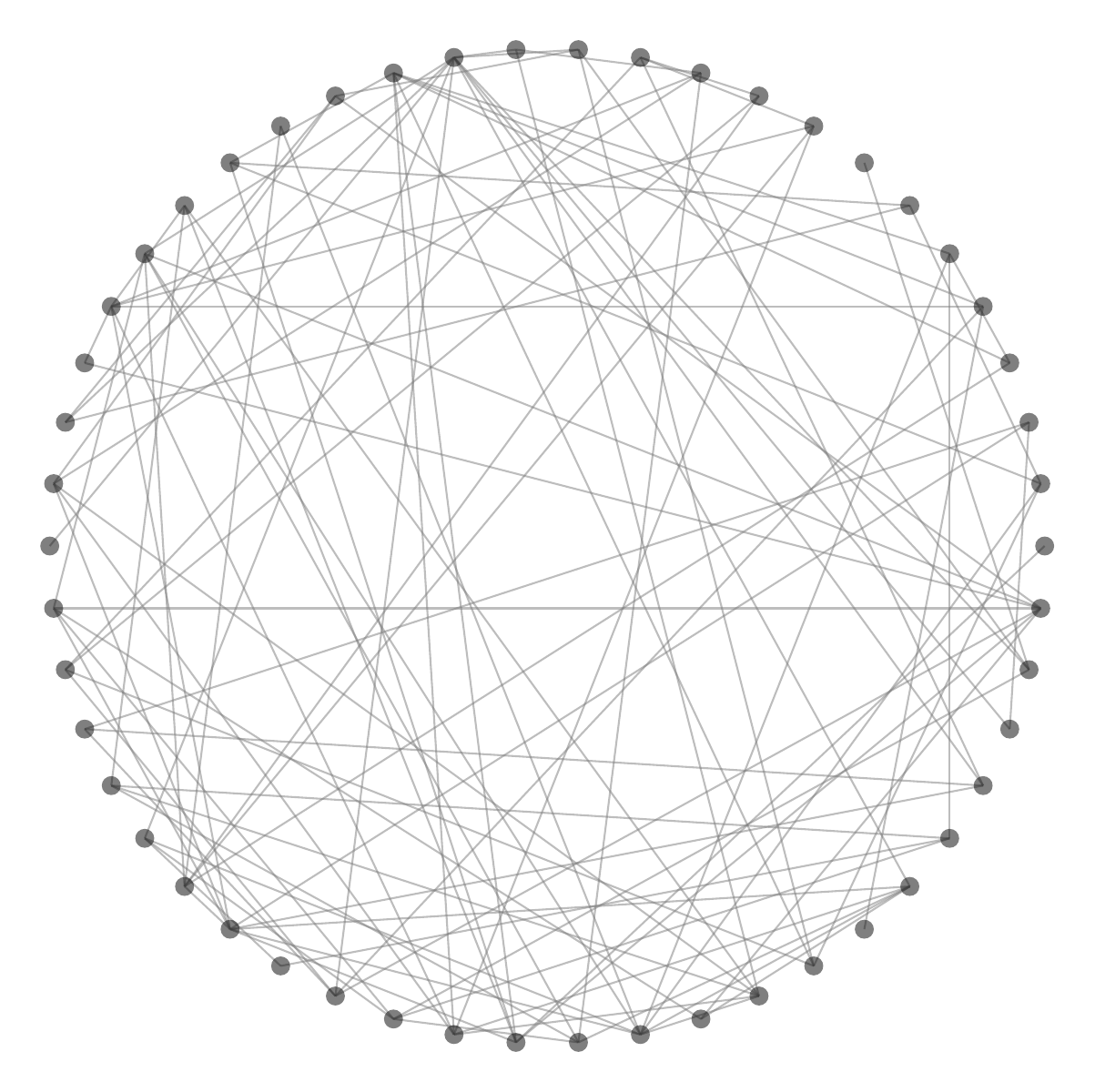}
\caption{50 nodes}
\label{fig:50nodes}
\end{subfigure}
\hspace{0.3cm}
\begin{subfigure}{0.22\textwidth}
\centering
\includegraphics[width=3.7cm]{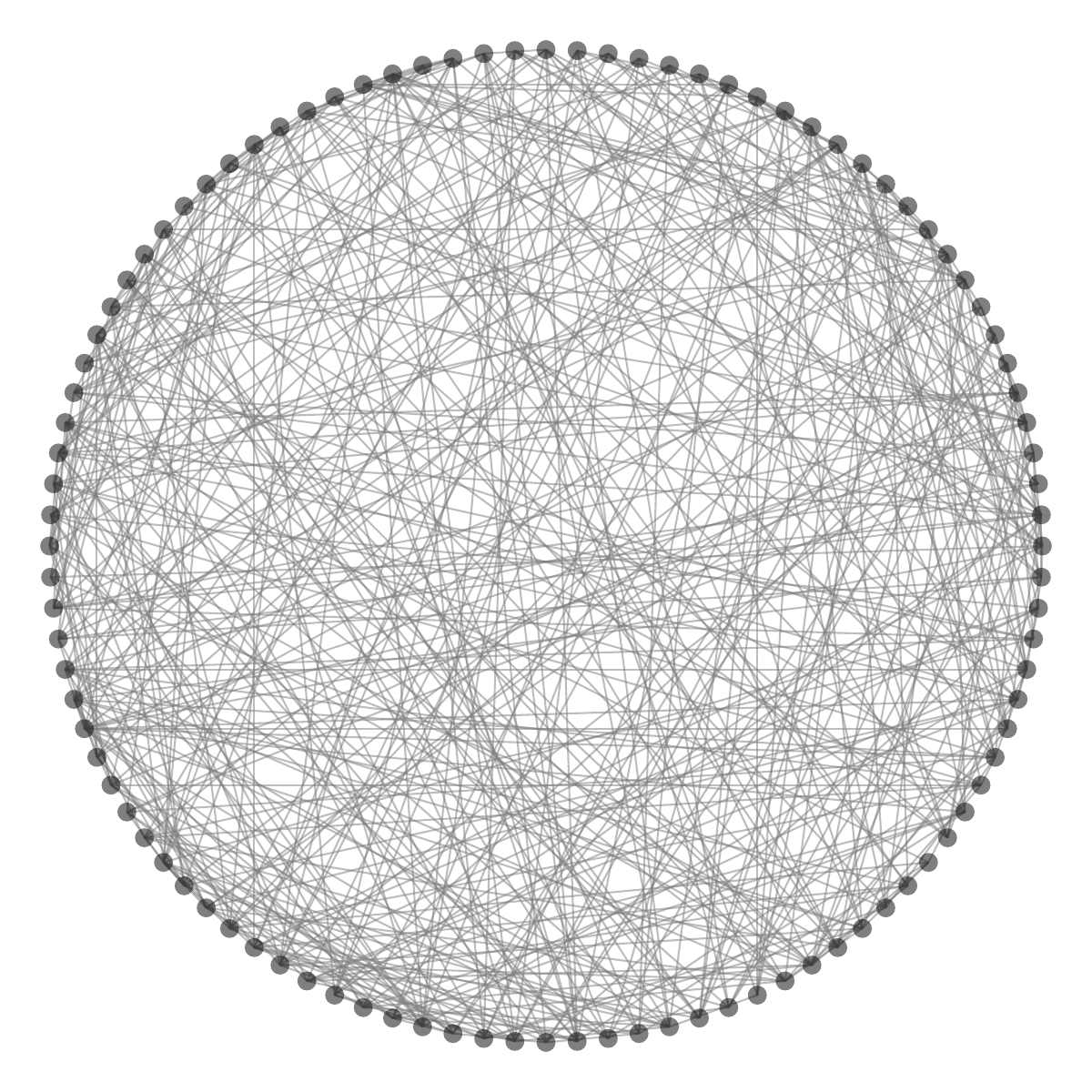}
\caption{100 nodes}
\label{fig:100nodes}
\end{subfigure}
\hspace{0.3cm}
\begin{subfigure}{0.22\textwidth}
\centering
\includegraphics[width=3.7cm]{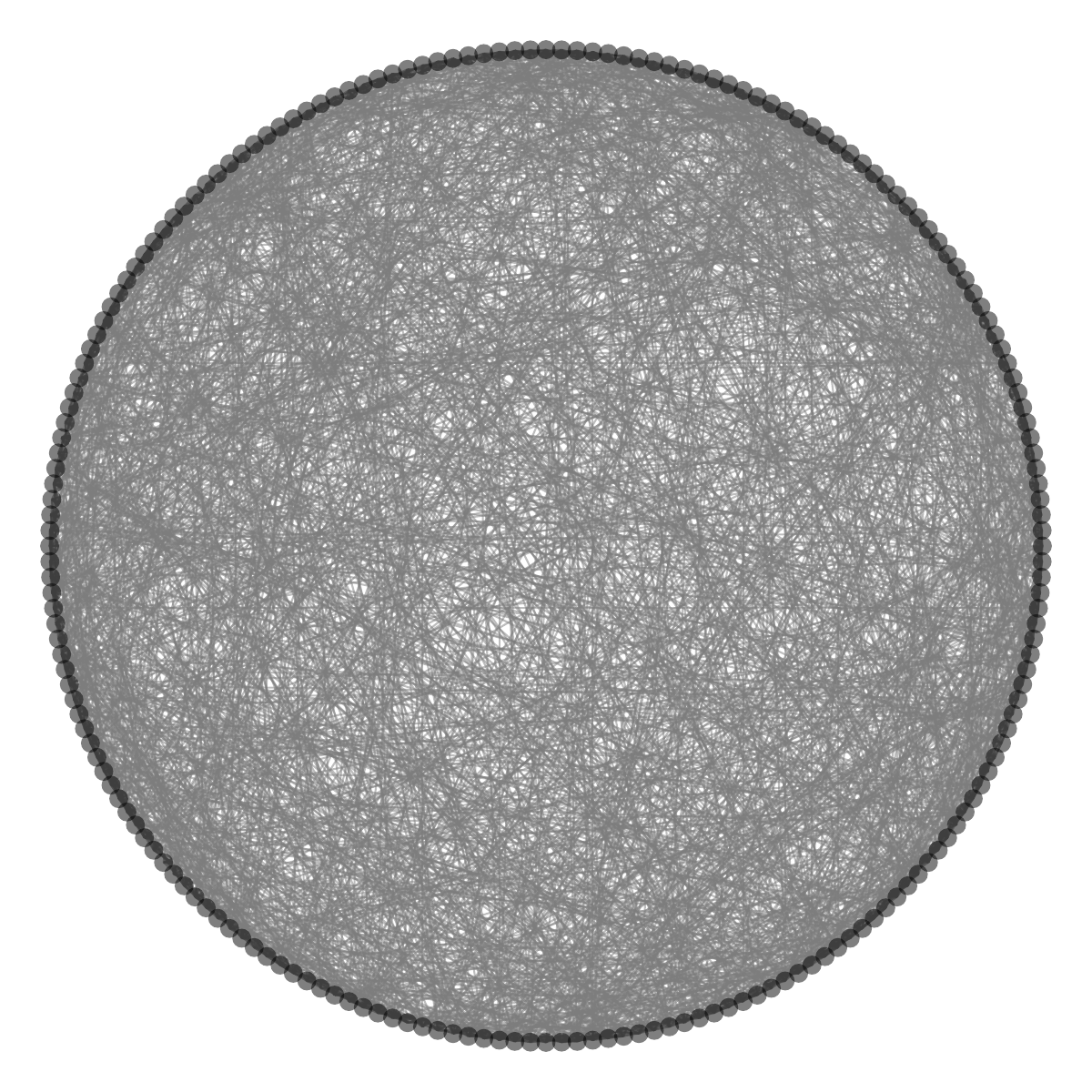}
\caption{200 nodes}
\label{fig:200nodes}
\end{subfigure}
\hspace{0.3cm}
\begin{subfigure}{0.22\textwidth}
\centering
\includegraphics[width=3.7cm]{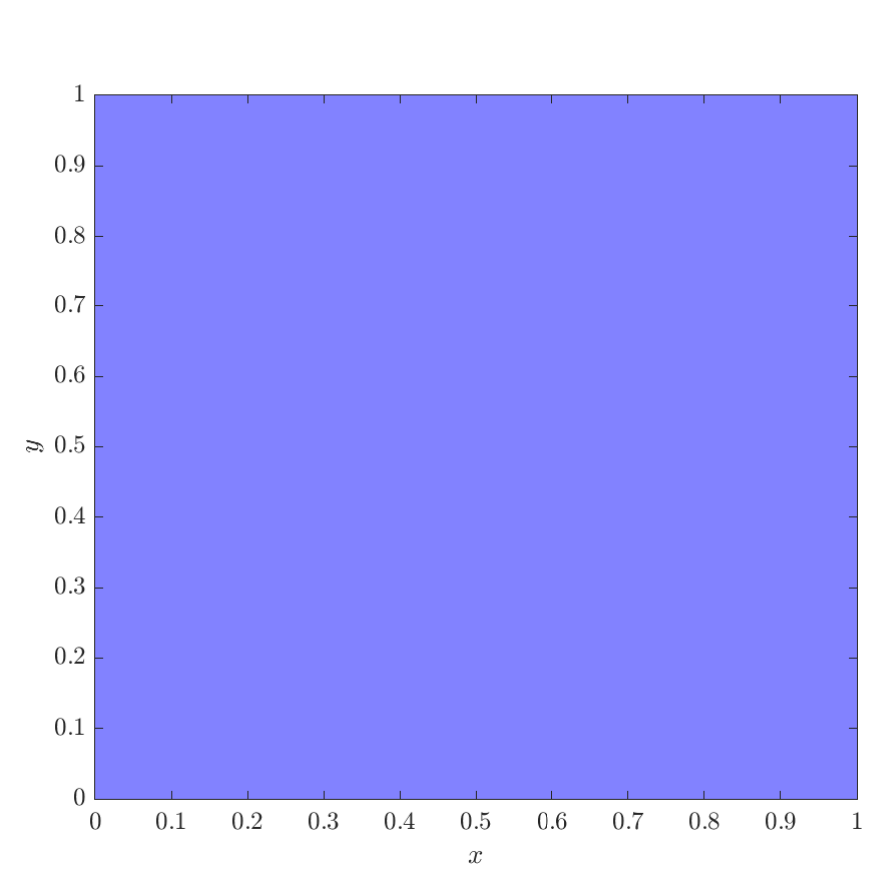}
\caption{Graphon}
\label{fig:er_graphon}
\end{subfigure}

\caption{Constant graphon with $\bbW(u,v)=0.4$ for all $u,v \in [0,1]$ (d). Sequence of stochastic graphs with 50 (a), 100 (b), and 200 nodes (c) which converges to this graphon.}
\label{fig:graphon_example}
\end{figure*}

A sequence of graphs converges to a graphon in the sense that densities of homomorphisms from certain graph motifs onto these graphs converge. Given an unweighted graph $\bbF = (\ccalV', \ccalE')$ and a (possibly weighted) graph $\bbG=(\ccalV,\ccalE,\bbS)$, a homomorphism between $\bbF$ and $\bbG$ is defined as a map $\beta: \ccalV' \to \ccalV$ such that, if $(i,j) \in \ccalE'$, $(\beta(i),\beta(j)) \in \ccalE$. Because there are multiple such maps between the nodes of $\bbF$ and $\bbG$, to count them we define the number of homomorphisms between $\bbF$ and $\bbG$ as
\begin{equation*}
\mbox{hom}(\bbF,\bbG) = \sum_\beta \prod_{(i,j)\in\ccalE'}[\bbS]_{\beta(i)\beta(j)}\ .
\end{equation*}
For unweighted graphs $\bbG$, $\mbox{hom}(\bbF,\bbG)$ is the total number of adjacency preserving maps between $\ccalV'$ and $\ccalV$. For weighted graphs, each map $\beta$ is weighted proportionally to the product of the weights of the edges $(\beta(i),\beta(j)) \in \ccalE$ for all $i, j \in \ccalV$.

Since homomorphisms are only a fraction of the total number of maps between the nodes of $\bbF$ and the nodes of $\bbG$, we can define a density of homomorphisms from $\bbF$ to $\bbG$. This density is denoted $t(\bbF,\bbG)$ and given by
\begin{equation*}
t(\bbF,\bbG) = \frac{\mbox{hom}(\bbF,\bbG)}{|\ccalV|^{|\ccalV'|}}\ .
\end{equation*}
The density $t(\bbF,\bbG)$ can be likened to the probability of sampling a homomorphism when picking a random map $\beta:\ccalV' \to \ccalV$. This allows extending the notion of density of homomorphisms to graphons. Explicitly, we define the density of homomorphisms from the graph $\bbF$ to the graphon $\bbW$ as
\begin{equation*} \label{eqn:hom_graphon}
t(\bbF,\bbW) = \int_{[0,1]^{|\ccalV'|}} \prod_{(i,j) \in \ccalE'} \bbW(u_i,u_j) \prod_{i \in \ccalV'} du_i
\end{equation*}
i.e., $t(\bbF,\bbW)$ is the probability of sampling $\bbF$ from $\bbW$.

A sequence of graphs $\{\bbG_n\}$ converges to $\bbW$ if and only if the density of homomorphisms between any finite, undirected and unweighted graph $\bbF = (\ccalV', \ccalE')$ and the $\bbG_n$ converges to the density of homomorphisms between $\bbF$ and $\bbW$. More formally, $\bbG_n \to \bbW$ if and only if
\begin{equation}\label{eqn_convergence}
   \lim_{n \to \infty} t(\bbF,\bbG_n) = t(\bbF, \bbW)
\end{equation}
for all finite, undirected and unweighted $\bbF$ \cite{lovasz2006limits}. We refer to graphs $\bbF$ as \textit{motifs} to emphasize that a graph sequence only converges when the densities of these structural motifs converge.

%


\subsection{Graphons as Graph Generative Models}

The other important interpretation of graphons is as generative models for graphs. A graph $\bbG_n = (\ccalV_n,\ccalE_n, \bbS_n)$ can be obtained from the graphon $\bbW$ in two steps. First, $n$ points $u_i \in [0,1]$ are chosen to be the labels of the nodes $i \in \ccalV_n$. Then, the edges $\ccalE_n$ and the adjacency matrix $\bbS_n$ are determined either by assigning weight $\bbW(u_i,u_j)$ to $(i,j)$, or by sampling the edge $(i,j)$ with probability $\bbW(u_i,u_j)$. There are many different strategies for choosing the node labels $u_i$ and, given the possibility of defining either weighted (deterministic) or stochastic (unweighted) edges $\ccalE_n$, there are many types of graphs  that can be generated in this way. We focus on three: template graphs, weighted graphs, and stochastic graphs.

\begin{definition}[Template graphs] \label{def:template}
Let $\{u_i\}_{i=1}^n$ be the regular $n$-partition of $[0,1]$,
\begin{equation*}
u_i = \frac{i-1}{n}
\end{equation*}
for $1 \leq i \leq n$. The $n$-node template graph~$\mbG_n$, whose GSO we denote~$\mbS_n$, is obtained from $\bbW$ as
\begin{equation*}
[\mbS_n]_{ij} = \bbW(u_i,u_j) \mbox{ for } 1 \leq i,j \leq n \text{.}
\end{equation*}
\end{definition}

Template graphs are the simplest type of graph that can be generated from a graphon. The node labels $u_i$ are determined by partitioning $[0,1]$ into $n$ intervals of equal length, and the edge weights $[\mbS_n]_{ij}$ are obtained by evaluating the graphon at $(u_i,u_j)$. Template graphs are always weighted and, if $\bbW$ is strictly positive, they are also complete (i.e., all nodes are connected to every node of the graph).

\begin{definition}[Weighted graphs] \label{def:weighted}
Let $\{u_i\}_{i=1}^n$ be $n$ points sampled independently and uniformly at random from $[0,1]$,
\begin{equation*}
u_i \sim \mbox{unif}(0,1)
\end{equation*}
for $1 \leq i \leq n$. The $n$-node weighted graph~$\mbG_{\tilde{n}}$, whose GSO we denote~$\mbS_{\tilde{n}}$, is obtained from $\bbW$ as
\begin{equation*}
[\mbS_{\tilde{n}}]_{ij} = \bbW(u_i,u_j) \mbox{ for } 1 \leq i,j \leq n \text{.}
\end{equation*}
\end{definition}

Weighted graphs are obtained by sampling $u_i$ uniformly at random from the unit interval and assigning weights $\bbW(u_i,u_j)$ to edges $(i,j)$. As the name implies, these graphs are weighted and, like template graphs, they can also be complete.

\begin{definition}[Stochastic graphs] \label{def:stochastic}
Let $\{u_i\}_{i=1}^n$ be $n$ points sampled independently and uniformly at random from $[0,1]$,
\begin{equation*}
u_i \sim \mbox{unif}(0,1)
\end{equation*}
for $1 \leq i \leq n$. The $n$-node stochastic graph~$\bbG_{\tilde{n}}$, whose GSO we denote~$\bbS_{\tilde{n}}$, is obtained from $\bbW$ as
\begin{equation*}
[\bbS_{\tilde{n}}]_{ij} = [\bbS_{\tilde{n}}]_{ji} \sim \mbox{Bernoulli}(\bbW(u_i,u_j)) \mbox{ for } 1 \leq i,j \leq n \text{.}
\end{equation*}
\end{definition}

Stochastic graphs are stochastic in both the node labels $u_i$, which are sampled uniformly at random, and in the edges $(i,j)$, which are sampled with probabilities $\bbW(u_i,u_j)$. Also called $\bbW$-random graphs \cite[Ch. 10.1]{lovasz2012large}, these graphs are always unweighted, i.e., $\bbS_{\tilde{n}} \in \{0,1\}^{n \times n}$. Note that stochastic graphs can also be obtained from weighted graphs as
\begin{equation} \label{eqn:stochastic_from_weighted}
[\bbS_{\tilde{n}}]_{ij} = [\bbS_{\tilde{n}}]_{ji} \sim \mbox{Bernoulli}([\mbS_{\tilde{n}}]_{ij}) \ .
\end{equation}
Thus, weighted graphs may be used as random graph models for $n$-node unweighted graphs.

We conclude this section by introducing the notion of a graphon induced by a graph, which will be useful in later derivations. Let $\bbG$ be a graph with GSO $\bbS \in \reals^{n \times n}$ and node labels $\{u_i\}_{i=1}^n$, $u_i \in [0,1]$. If the node labels are unknown, simply define $u_i = (i-1)/n$ for $1 \leq i \leq n$. Further define the intervals $I_i = [u_i,u_{i+1})$ for $1 \leq i \leq n-1$ and $I_n = [u_n,1]\cup[0,u_1)$. The graphon $\bbW_\bbG$ induced by $\bbG$ is given by
\begin{align} \label{eqn:graphon_ind}
{\bbW}_\bbG(u,v) = \sum_{i=1}^n\sum_{j=1}^{n}[{\bbS}]_{ij}\mbI(u \in I_i) \mbI(v \in I_j)
\end{align}
where $\mbI$ is the indicator function.

%% file: filter-transferability.tex

 
In the following, we discuss the transferability properties of graph filters. We begin by defining the limits of graph filters---graphon filters---in Section \ref{sbs:graphon-filters}. Then, in Section \ref{sbs:filter-transf} we interpret graphon filters as generative models for graph filters, and derive an upper bound to the transferability or approximation error between a graphon filter and a graph filter with the same coefficients on a generic graph sampled from the graphon (Theorem \ref{lemma:graphon-graph-filter}). To illustrate this result with concrete examples, in Propositions \ref{lemma:graphon-graph-filter-template}--\ref{thm:graphon-graph-filter-stochastic} we particularize Theorem \ref{lemma:graphon-graph-filter} to template, weighted and stochastic graphs. Sticking with the stochastic graph example, in Proposition \ref{thm:graph-graph-filter-stochastic} we then show that the error of transferring a graph filter across graphs associated with the same graphon can also be upper bounded as an immediate consequence of the graph filter approximation result.
 

\subsection{Graphon Filters} \label{sbs:graphon-filters} 

The natural limit of signals supported on graphs that are sampled from or converge to a graphon is the graphon signal.
A graphon signal is a function $X \in L_2([0,1])$ and, like the graphon, it is both the limit of a convergent sequence of graph signals and a generative model for graph signals on graphs instantiated from a graphon \cite{ruiz2020graphonsp}. For instance, given a template graph $\mbG_n$ generated from the graphon $\bbW$ as in Definition \ref{def:template}, the graphon signal $X$ can be used to generate the graph signal
\begin{equation} \label{eqn:template_graph_signal}
[\mathbb{x}_n]_i = X\left(\frac{i-1}{n}\right)
\end{equation}
for $1 \leq i \leq n$. Similarly, on the weighted graph $\mbG_{\tilde{n}}$ and on the stochastic graph $\bbG_{\tilde{n}}$ [cf. Definitions \ref{def:weighted} and \ref{def:stochastic}], $X$ can be used to generate equivalent signals $\mathbb{x}_{\tilde{n}}$ and $\bbx_{\tilde{n}}$ defined as
\begin{equation} \label{eqn:weighted_graph_signal}
[\mathbb{x}_{\tilde{n}}]_i = [\bbx_{\tilde{n}}]_i  = X\left(u_i\right)
\end{equation}
for $1 \leq i \leq n$, where the $u_i$ are sampled independently and uniformly at random from $[0,1]$.

Graphon signals can take any form, but the most straightforward ones are those induced by graph signals. The graphon signal induced by a graph signal $\bbx_n \in \reals^n$ is a step function given by \cite{ruiz2019graphon}
\begin{equation} \label{eqn:graphon_signal_ind}
X_n(u) = \sum_{i=1}^n [\bbx_n]_i\ \mbI(u \in I_i)
\end{equation}
where $I_i = [u_i,u_{i+1})$ for $1 \leq i \leq n-1$ and $I_n = [u_n,1]\cup[0,u_1)$. When the node labels are not defined, we set $u_i = (i-1)/n,i/n)$ for $1 \leq i \leq n$.

The diffusion operator for graphon signals, denoted $T_\bbW$, is the integral operator with kernel $\bbW$, i.e.,
\begin{equation} \label{eqn:graphon_shift}
T_\bbW X(v) = \int_0^1 \bbW(u,v)X(u)du \text{.}
\end{equation}
In analogy with the GSO, we refer to this operator as the graphon shift operator (WSO) of $\bbW$.

Because graphons are bounded and symmetric, the WSO is a self-adjoint Hilbert-Schmidt (HS) operator. This allows expressing $\bbW$ in the operator's spectral basis as $\bbW(u,v) = \sum_{i \in \mbZ\setminus \{0\}} \lambda_i \varphi_i(u)\varphi_i(v)$ and rewriting $T_\bbW$ as 
\begin{equation} \label{eqn:graphon_spectra}
T_\bbW X(v) = \sum_{i \in \mbZ\setminus \{0\}}\lambda_i \varphi_i(v) \int_0^1 \varphi_i(u)X(u)du
\end{equation}
where the eigenvalues $\lambda_i$, $i \in \mbZ\setminus \{0\}$, follow the ordering $1 \geq \lambda_1 \geq \lambda_2 \geq \ldots \geq \ldots \geq \lambda_{-2} \geq \lambda_{-1} \geq -1$. Since $T_\bbW$ is compact, the eigenvalues accumulate around $0$ as $|i| \to \infty$ \cite[Thm. 3, Chapter 28]{lax02-functional}. This is illustrated in Fig. \ref{fig:eigenvalues_and_filter}.

Similarly to how the graph eigenvalues are interpreted as graph frequencies, we interpret the graphon eigenvalues as \textit{graphon frequencies} and the graphon eigenfunctions as the graphon's \textit{oscillation modes}. High magnitude eigenvalues represent the low frequencies, and vice-versa.
Since the eigenfunctions $\varphi_i$ form an orthonormal basis of $L_2([0,1])$, we can define the graphon Fourier transform (WFT) of a graphon signal $X$ as its projection onto this eigenbasis. Explicitly, the WFT of $X$ is denoted $\hat{X}$ and given by \cite{ruiz2019graphon}
\begin{equation} \label{eqn:gft}
[\hat{X}]_i = \int_0^1 X(u) \varphi_i(u) du \text{.}
\end{equation}
Conversely, the inverse graphon Fourier transform or iWFT is $X = \sum_{i \in \mbZ \setminus \{0\}}[\hat{X}]_i\varphi_i$.

A graphon filter is an operator $T_{\bbH}: L_2([0,1]) \to L_2([0,1])$ which maps graphon signals $X$ to $Y=T_\bbH X$. In particular, the graphon shift operator allows defining LSI graphon filters given by
\begin{align}\begin{split} \label{eqn:lsi-wf}
&T_\bbH X = \sum_{k=0}^{K-1} h_k T_{\bbW}^{(k)} X \quad \mbox{with} \\
&T_{\bbW}^{(k)}X = \int_0^1 \bbW(u,v)T_\bbW^{(k-1)} X(u)du
\end{split}\end{align}
where $T_{\bbW}^{(0)} = \bbI$ is the identity operator and $\bbh = [h_0, \ldots, h_{K-1}]$ are the filter coefficients. This filter is shift-invariant because $T_\bbH$ and $T_\bbW$ commute. I.e., if we shift the input as $X' = T_\bbW X$, the output is equivalently shifted as $Y' = T_\bbH X' = T_\bbH T_\bbW X = T_\bbW T_\bbH X = T_\bbW Y$. LSI graphon filtering is also called graphon convolution \cite{ruiz2020graphonsp}.

Using the spectral decomposition of $T_\bbW$ in \eqref{eqn:graphon_spectra}, $T_\bbH$ can be written as
\begin{align} \label{eqn:spec-graphon_filter}
\begin{split}
T_\bbH X(v) &= \sum_{i \in \mbZ\setminus \{0\}} \sum_{k=0}^{K-1} h_k \lambda_i^k \varphi_i(v) \int_0^1 \varphi_i(u)X(u)du\\
&= \sum_{i \in \mbZ\setminus \{0\}} h(\lambda_i) \varphi_i(v) \int_0^1 \varphi_i(u)X(u)du \text{.}
\end{split}
\end{align}
This expression highlights the filter's spectral representation, which is given by $h(\lambda) = \sum_{k=0}^{K-1} h_k \lambda^k$. Because the WFT of the filter is a polynomial on the graphon eigenvalues, note that as $K\to \infty$ LSI graphon filters can be used to approximate any graphon filter with spectral representation $h(\lambda)=f(\lambda)$ where $f$ is analytic. This is a consequence of the Cayley-Hamilton theorem in the limit of matrices. Moreover, unlike the WFT of the signals $X$ and $Y$, the spectral representation of the LSI graphon filter does not depend on the graphon eigenfunctions---only on the graphon eigenvalues and the coefficients $\bbh$. This is illustrated in Fig. \ref{fig:eigenvalues_and_filter} where the red curve represents $h(\lambda)$ and the blue lines are the eigenvalues of $\bbW$.

The fact that the spectral response of the LSI graphon filter only depends on the eigenvalues of the graphon and on the filter coefficients is something that these filters have in common with LSI graph filters. This can be seen by comparing \eqref{eqn:spec-lsi-gf} and \eqref{eqn:spec-graphon_filter}. Indeed, if the coefficients $h_k$ in these equations are the same, the spectral responses of the LSI graph filter and of the LSI graphon filter are determined by the same function $h(\lambda)$. Given a graph $\bbG_n$ (with GSO $\bbS_n$) and a graphon $\bbW$, the only thing that changes is where this function is evaluated: at the graph eigenvalues $\{\lambda_i(\bbS_n)\}_{i \in\ccalL \subset \mbZ \setminus \{0\}}$, or at the graphon eigenvalues $\{\lambda_i(T_\bbW)\}_{i \in \mbZ \setminus \{0\}}$ [cf. Fig. \ref{fig:eigenvalues_and_filter}].
The implication is that graphon filters can be used as generative models for graph filters as long as the coefficients $h_k$ in \eqref{eqn:spec-lsi-gf} and \eqref{eqn:spec-graphon_filter} are the same.

For example, let $Y = T_\bbH X$ be a LSI graphon filter which we represent, for simplicity, as the map $Y=\Phi(X;\bbh,\bbW)$ to emphasize the dependence on the coefficients $\bbh$ and the graphon $\bbW$. We can generate a graph filter $\bby_n = \Phi(\bbx_n;\bbh,\bbS_n)$ from this graphon filter by instantiating a graph $\bbG_n$ from $\bbW$ as in Definitions \ref{def:template}, \ref{def:weighted} or \ref{def:stochastic}, and the corresponding graph signal $\bbx_n$ as in \eqref{eqn:template_graph_signal} or \eqref{eqn:weighted_graph_signal}. Generating graph filters from graphon filters is useful because it allows designing filters for graphons and \textit{transferring} them to graphs. Since graphons are limits of convergent graph sequences, it also justifies transferring filters across graphs obtained from the same graphon.

Before delving into the transferability analysis of graph filters, we introduce the definition of graphon filters induced by graph filters, which will allow comparing graph and graphon filters directly. The graphon filter induced by the graph filter $\Phi (\bbx_n; \bbh, \bbS_n/n)=\sum_{k=0}^{K-1} h_k \bbS_n^k\bbx_n/n^k$ is given by
\begin{align}\begin{split} \label{eqn:graphon_filter_ind}
&\Phi (X_n; \bbh, \bbW_n)= \sum_{k=0}^{K} h_k T_{\bbW_n}^{(k)} X_n(v) \quad \mbox{with} \\
&T_{\bbW_n}^{(k)}X_n(v) = \int_0^1 \bbW_n(u,v)T_{\bbW_n}^{(k-1)} X_n(u)du
\end{split}\end{align}
where the graphon $\bbW_n = \bbW_{\bbG_n}$ is the graphon induced by $\bbG_n$ \eqref{eqn:graphon_ind} and $X_n$ is the graphon signal induced by the graph signal $\bbx_n$ \eqref{eqn:graphon_signal_ind}. Note that in the graph filter $\bbPhi(\bbx_n;\bbh,\bbS_n/n)$, we consider the normalized shift operator $\bbS_n/n$. This is important to guarantee that the output of the graph filter and its corresponding induced graphon filter match, as shown by the following lemma. The proof is deferred to Appendix \ref{sec:appendix0} in the supplementary material.

\begin{lemma} \label{lemma:induced_filters}
Let $\bby_n = \Phi (\bbx_n; \bbh, \bbS_n/n)=\sum_{k=0}^{K-1} h_k \bbS_n^k\bbx_n/n^k$ be a graph filter, and let $Y_n = \Phi (X_n; \bbh, \bbW_n)$ be its corresponding induced graphon filter \eqref{eqn:graphon_filter_ind}. Then, it holds that 
\begin{equation}
 Y_n(u) = \sum_{i=1}^n[\bby_n]_i \mbI(u \in I_i)   
\end{equation}
where $I_i = [u_i,u_{i+1})$ for $1 \leq i \leq n-1$ and $I_n = [u_n,1]\cup[0,u_1)$, and $u_i$ are the node labels associated with $\bbG_n$. When the node labels are not defined, $u_i = (i-1)/n$ for $1 \leq i \leq n$.
\end{lemma}


\subsection{Graph Filter Transferability} \label{sbs:filter-transf}

\begin{figure*}
\centering
\input{figures/figures/1030/30_filter_small_eig.tex}
\caption{Graphon eigenvalues (blue) and graph eigenvalues (green) for a graph $\bbG_n$ taken from a sequence converging to the graphon $\bbW$. For a fixed set of parameters $h_k$ [cf. \eqref{eqn:spec-lsi-gf} and \eqref{eqn:spec-graphon_filter}], the red curve represents a filter's frequency or spectral response. The graphon filter with coefficients $h_k$ has frequency response given by the blue dots. The graph filter with coefficients $h_k$ has frequency response given by the green dots. Note that, to have transferability (Thm. \ref{lemma:graphon-graph-filter}), the frequency response must vary slowly for eigenvalues with magnitude close to zero. This is indicated by the shaded red rectangle from -0.1 to 0.1.}
\label{fig:eigenvalues_and_filter}
\end{figure*}
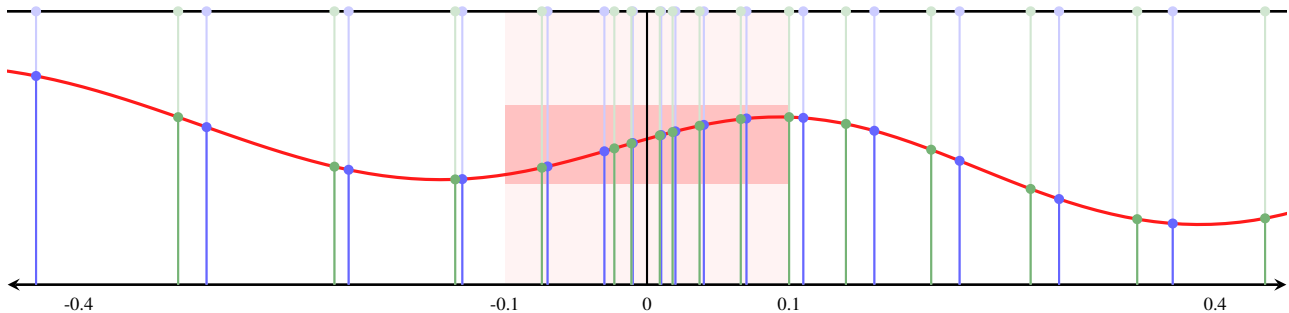


The graph filter (and GNN) transferability results require introducing two definitions: the $c$-band cardinality and the $c$-eigenvalue margin, which are both graphon spectral properties associated with a fixed eigenvalue threshold $c$.

\begin{definition}[$c$-band cardinality of $\bbW$] \label{def:c_band_card}
The $c$-band cardinality of a graphon $\bbW$, denoted $B^c_{\bbW}$, is the number of eigenvalues $\lambda_i$ of $T_\bbW$ with absolute value larger or equal to $c$, i.e., $$B^c_{\bbW} = \#\{\lambda_i\ :\ |\lambda_i| \geq c\}.$$
\end{definition}

\begin{definition}[$c$-eigenvalue margin of $\bbW$ and $\bbW'$] \label{def:c_eig_margin}
For two graphons $\bbW$ and $\bbW'$, the $c$-eigenvalue margin, denoted $\delta^c_{\bbW\bbW'}$, is given by $$\delta^c_{\bbW\bbW'} = \min_{i, j\neq i} \{ |\lambda_i(T_{\bbW'}) - \lambda_j(T_\bbW)|\ :\ |\lambda_i(T_{\bbW'})| \geq c\}$$ where $\lambda_i(T_{\bbW'})$ and $\lambda_i(T_\bbW)$ denote the eigenvalues of $T_\bbW'$ and $T_\bbW$ respectively.
\end{definition}

We also introduce the following Lipschitz continuity assumptions on the
the filter spectral response.

\begin{assumption} \label{as2}
The spectral response of the convolutional filter $h$ is $A_h$-Lipschitz in $[-1,-c] \cup [c,1]$ and $a_h$-Lipschitz in $(-c,c)$, with $a_h < A_h$ Moreover, $|h(\lambda)|<1$.
\end{assumption} 

Under assumption AS\ref{as2}, we can prove the following theorem, which states that a graphon filter can be approximated by a graph filter on a large graph sampled from this graphon.

\begin{theorem}[Graphon filter transferability to a generic graph] \label{lemma:graphon-graph-filter}
Let $Y=\Phi (X; \bbh, \bbW)$ be a graphon filter [cf. \eqref{eqn:spec-graphon_filter}] satisfying assumption AS\ref{as2}. Given a generic graph $\bbG_n$ with GSO $\bbS_n$ sampled from $\bbW$, let $\bby_n = \Phi(\bbx_n;\bbh,\bbS_n/n)$ be a graph filter instantiated from $\Phi (X; \bbh, \bbW)$ on this graph. For any $0 < c \leq 1$, it holds that
\begin{align} \label{eqn:thm4result1}
\begin{split}
\|Y_{n}-Y\| \leq \bigg(A_h &+ \frac{\pi B_{\bbW_n}^c}{\delta^c_{\bbW\bbW_n}}\bigg)\|\bbW-\bbW_n\|\|X\| \\
&+ (A_h c + 2)\|X-X_n\| + 2a_h c\|X\|
\end{split}
\end{align}
where $\bbW_n := \bbW_{\bbG_n}$ is the graphon induced by $\bbG_n$ and $Y_{n} = \Phi (X_{n}; \bbh, \bbW_n)$ is the graphon filter induced by $\bby_n = \Phi (\bbx_n; \bbh, \bbS_n/n)$ \eqref{eqn:graphon_filter_ind}.
\end{theorem}
\begin{proof}
Refer to Appendix \ref{sec:appendixA}.
\end{proof}

Thm. \ref{lemma:graphon-graph-filter} shows that the error incurred when transferring a graphon filter to a graph sampled from the graphon is upper bounded by the sum of three terms. The first depends on the distance between the graph and the graphon, i.e., it is the graph sampling error. Since sequences of graphs sampled from a graphon converge to it\footnote{Sequences of graphs converging to a graphon converge in the cut norm up to node permutations \cite[Chapter 11.7]{lovasz2012large}, and cut norm convergence implies convergence in $L^2$ (see \cite[Prop. 4]{ruiz2020graphonsp}).}, we can expect this distance to be small for large $n$, and to eventually vanish as $n\to\infty$. Disregarding the norm of the input signal $X$ (it is finite), this term also depends on what we call the filter's \textit{transferability constant}.
By being proportional to $A_h$ and $B^c_{\mbW_n}$, and inversely proportional to $\delta^c_{\bbW\mbW_n}$, this constant is directly proportional to the filter variability, so the smaller it is, the more transferable the filter. Note that although the bound in Thm. \ref{lemma:graphon-graph-filter} is somewhat implicit due to the dependence of $\delta^c_{\bbW\mbW_n}$ on $n$, for a given graph $\bbG_n$ and graphon $\bbW_n$ this quantity can always be calculated, and for large $n$ it approaches the smallest graphon eigengap in the interval $[-1,-c]\cup[c,1]$.

The second term of the bound is related to the distance between the graphon signal and the graph signal. I.e., it stems from sampling the signal. Since $X$ is in $L^2$, $X_n$ converges up to node permutations. Hence, $\|X_n-X\|$ vanishes as $n\to\infty$.
Note that this term also depends on $a_h c$, which upper bounds the maximum filter amplification in a part of the eigenvalue spectrum determined by a threshold $c$---the band $(-c, c)$, see Fig. \ref{fig:eigenvalues_and_filter}. Since $a_h < A_h$, the contribution of this term to the approximation error is small.

The constant $c$ also appears in the third term of the bound, $2 a_h c \|X\|$, which is related to the output signal's non-transferable energy for eigenvalues smaller than $c$.
For a given value of $c$, spectral components associated with $\lambda \in [-1,-c] \cup [c,1]$ (i.e., low-frequency components) are completely transferable in the sense that, if the signal $X$ or the filter $\Phi (X; \bbh, \bbW)$ are $c$-bandlimited (i.e., spectral components associated with eigenvalues $|\lambda| \in [0,c)$ are zero), the transferability bound is only given by the first two terms of \eqref{eq:prop1} and thus vanishes with the size of the graph. In the more general case in which the input signal is not bandlimited, regardless of how we fix $c \in (0,1]$ there will be some residual non-transferable energy associated with the spectral components $|\lambda| \in [0,c)$. However, the effect of the non-transferable spectral components on the transferability error is attenuated by the filter since its Lipschitz constant on the $(-c,c)$ interval, $a_h$, is smaller than $A_h$, the Lipschitz constant on $[-1,-c]\cup[c,1]$. In particular, if $a_h \ll A_h$ the filter resembles a filter with constant frequency response on the $(-c,c)$ band. See Sec. \ref{sbs:discussion} for more in-depth discussion about this non-transferable energy and filters with constant band.


Next, we show how Thm. \ref{lemma:graphon-graph-filter} particularizes to the simplest type of graph that can be sampled from a graphon---the template graph in Definition \ref{def:template}. To do so, we need two additional assumptions on the Lipschitz continuity of the graphon and the graphon signal respectively.

\begin{assumption} \label{as1}
The graphon $\bbW$ is $A_w$-Lipschitz, i.e., $|\bbW(u_2,v_2)-\bbW(u_1,v_1)| \leq A_w(|u_2-u_1|+|v_2-v_1|)$.
\end{assumption} 
\begin{assumption} \label{as3}
The graphon signal $X$ is $A_x$-Lipschitz.
\end{assumption}

\begin{proposition}[Graphon filter transferability to a template graph] \label{lemma:graphon-graph-filter-template}
Let $Y=\Phi (X; \bbh, \bbW)$ be a graphon filter [cf. \eqref{eqn:spec-graphon_filter}] satisfying assumptions AS\ref{as2}--AS\ref{as3}. Given a template graph $\mbG_n$ with GSO $\mbS_n$ [cf. Definition \ref{def:template}], let $\mathbb{y}_n = \Phi(\mathbb{x}_n;\bbh,\mbS_n/n)$ be a graph filter instantiated from $\Phi (X; \bbh, \bbW)$ on this graph. For any $0 < c \leq 1$, it holds that
\begin{align} \label{eq:prop1}
\begin{split}
\|\mbY_{n}-Y\| \leq \bigg(A_h &+ \frac{\pi B_{\mbW_n}^c}{\delta^c_{\bbW\mbW_n}}\bigg)\frac{2A_w}{n}\|X\| \\
&+ \frac{A_x(a_h c+2)}{n} + 2a_h c\|X\|
\end{split}
\end{align}
where $\mbW_n = \bbW_{\mbG_n}$ is the graphon induced by $\mbG_n$ and $\mbY_{n} = \Phi (\mbX_{n}; \bbh, \mbW_n)$ is the graphon filter induced by $\mathbb{y}_n = \Phi (\mathbb{x}_n; \bbh, \mbS_n/n)$ \eqref{eqn:graphon_filter_ind}.
\end{proposition}
\begin{proof}\renewcommand{\qedsymbol}{}
Refer to Appendix \ref{sec:appendixA.2} in the supplementary material.
\end{proof}

From Prop. \ref{lemma:graphon-graph-filter-template}, we see that, for template graphs, the parts of the graphon filter approximation bound associated with the graph and the graph signal sampling errors decrease as $n^{-1}$. However, the dependence on the Lipschitz constants $A_w$ and $A_x$ implies that graphon filters are harder to approximate when the graphon and the graphon signal are not smooth.


Using concentration inequalities for the uniform distribution, Prop. \ref{lemma:graphon-graph-filter-template} can be extended to weighted graphs. 
First, we need an additional definition---node stochasticity---which arises when the node labels are no longer given by a regular partition, but sampled uniformly at random.

\begin{definition}[Node stochasticity] \label{def:node_stoc}
For a fixed probability $\chi \in [0,1]$, the node stochasticity constant on $n$ nodes, denoted $\alpha(\chi, n)$, is defined as 
$$\alpha(n,\chi)=\log{({{(n+1)}^2}/{\log{(1-\chi)^{-1}}})} \text{.}$$
\end{definition}

\begin{proposition}[Graphon filter transferability to a weighted graph] \label{lemma:graphon-graph-filter-weighted}
Let $Y=\Phi (X; \bbh, \bbW)$ be a graphon filter [cf. \eqref{eqn:spec-graphon_filter}] satisfying assumptions AS\ref{as2}--AS\ref{as3}. Given a weighted graph $\mbG_{\tilde{n}}$ with GSO $\mbS_{\tilde{n}}$ [cf. Definition \ref{def:weighted}], let $\mathbb{y}_{\tilde{n}} = \Phi(\mathbb{x}_{\tilde{n}};\bbh,\mbS_{\tilde{n}}/n)$ be a graph filter instantiated from $\Phi (X; \bbh, \bbW)$ on this graph. Let $\chi_1, \chi_2 \in (0,0.3]$. For any $0 < c \leq 1$ and $n \geq 4/\chi_2$, with probability at least $[1-2\chi_1]\times[1-\chi_2]$ it holds that
\begin{align} \label{eq:prop2}
\begin{split}
\|\mbY_{\tilde{n}}-Y\| \leq &\bigg(A_h + \frac{\pi B_{\mbW_{\tilde{n}}}^c}{\delta^c_{\bbW\mbW_{\tilde{n}}}}\bigg)\frac{2A_w\alpha (n,\chi_1)}{n}\|X\| \\
&+ \frac{A_x\alpha (n,\chi_1)(a_h c+2)}{n} + 2a_h c\|X\|
\end{split}
\end{align}
\mbox{where $\mbW_{\tilde{n}} = \bbW_{\mbG_{\tilde{n}}}$ is the graphon induced by $\mbG_{\tilde{n}}$ and $\mbY_{\tilde{n}} = $}  \mbox{$\Phi (\mbX_{\tilde{n}};\bbh, \mbW_{\tilde{n}})$ is the graphon filter induced by $\mathbb{y}_{\tilde{n}} = \Phi (\mathbb{x}_{\tilde{n}}; \bbh,$} $\mbS_{\tilde{n}}/n)$ \eqref{eqn:graphon_filter_ind}.
\end{proposition}
\begin{proof}\renewcommand{\qedsymbol}{}
Refer to Appendix \ref{sec:appendixA.2} in the supplementary material.
\end{proof}

Therefore, graphon filters are also transferable to weighted graphs. The main difference between Prop. \ref{lemma:graphon-graph-filter-weighted} and Prop. \ref{lemma:graphon-graph-filter-template} is that the first and second terms of the transferability bound are now multiplied by the node stochasticity constant $\alpha(n,\chi_1)$ [cf. Definition \ref{def:node_stoc}]. This constant shows up because of the randomness associated with the node labels $u_i$, which are sampled uniformly at random. Note that  $\alpha(n,\chi_1)$ depends on both the graph size and $\chi_1$. The value of $\chi_1$ determines the confidence of the transferability bound. This confidence, given by $[1-2\chi_1]\times[1-\chi_2]$, also depends on the parameter $\chi_2$. Although $\chi_2$ does not appear in \eqref{eq:prop2}, there is an interplay between the confidence bound and the minimum graph size, since Prop. \ref{lemma:graphon-graph-filter-weighted} holds with probability $[1-2\chi_1]\times[1-\chi_2]$ for $n \geq 4/\chi_2$.


Next, we extend the graphon-graph filter approximation result to its more general form---graphon filter approximation by graph filters supported on stochastic graphs. To do so, we need to define edge stochasticity and make one additional assumption on the size of the graph. 

\begin{definition}[Edge stochasticity] \label{def:edge_stoc}
For a fixed probability $\chi \in [0,1]$, the edge stochasticity constant on $n$ nodes, denoted $\beta(\chi, n)$, is defined as 
$$\beta(n, \chi) = \sqrt{n\log{(2n/\chi)}} \text{.}$$
\end{definition}

Assumption \ref{as4} imposes a restriction on $n$ related to the graphon variability $A_w$ and its maximum degree $d_\bbW$.

\begin{assumption} \label{as4}
Given $\chi_3 \in (0,1)$, $n$ is such that
\begin{equation*}
n - \frac{\log{2n/\chi_3}}{d_\bbW} > 2\frac{A_w}{d_\bbW}
\end{equation*}
where $d_\bbW = \max_u \int_0^1 \bbW(u,v)dv$.
\end{assumption}

We also need the following lemma, which upper bounds the transferability error of a graph filter across a weighted graph and a stochastic graph that is sampled from it (see \eqref{eqn:stochastic_from_weighted}). 

\begin{lemma}[Graph filter transferability from weighted to stochastic graphs] \label{lemma:transf-weighted-stochastic}
Consider a graphon $\bbW$ satisfying assumption AS\ref{as1}, and let $\mbG_{\tilde{n}}$ be a weighted graph with GSO $\mbS_{\tilde{n}}$ [cf. Definition \ref{def:weighted}] and $\bbG_{\tilde{n}}$ a stochastic graph with GSO $\bbS_{\tilde{n}}$ obtained from $\mbS_{\tilde{n}}$ as in \eqref{eqn:stochastic_from_weighted}. Given a graphon signal $X$ satisfying assumption AS\ref{as3}, let $\mathbb{y}_{\tilde{n}} = \Phi(\mathbb{x}_{\tilde{n}}; \bbh, \mbS_{\tilde{n}}/n)$ and $\bby_{\tilde{n}} = \Phi(\bbx_{\tilde{n}}; \bbh, \bbS_{\tilde{n}}/n)$ be graph filters satisfying assumption AS\ref{as2} and acting on the graph signals $\mathbb{x}_{\tilde{n}}$ and $\bbx_{\tilde{n}}$ instantiated from $X$ on $\mbG_{\tilde{n}}$ and $\bbG_{\tilde{n}}$ respectively (note that $\mathbb{x}_{\tilde{n}}=\bbx_{\tilde{n}}$). Let $\chi_3 \in (0,1)$. For any $0 < c \leq 1$ and $n$ satisfying assumption AS\ref{as4}, with probability at least $1-\chi_3$ it holds that
\begin{align*}
\|\mathbb{y}_{\tilde{n}}-\bby_{\tilde{n}}\| \leq \Bigg(A_h + \frac{\pi B^c_{\bbW_{\tilde{n}}}}{\delta^c_{\mbW_{\tilde{n}}\bbW_{\tilde{n}}}}\Bigg)2\sqrt{\frac{\log{(2n/\chi_3)}}{n}}\|\bbx_{\tilde{n}}\| \\ + 2a_h c \|\bbx_{\tilde{n}}\|
\end{align*}
where $\mbW_{\tilde{n}} = \bbW_{\mbG_{\tilde{n}}}$ and $\bbW_{\tilde{n}} = \bbW_{\bbG_{\tilde{n}}}$ are the graphons induced by $\mbG_{\tilde{n}}$ and $\bbG_{\tilde{n}}$ respectively.
\end{lemma}
\begin{proof}\renewcommand{\qedsymbol}{}
Refer to Appendix \ref{sec:appendixB} in the supplementary material.
\end{proof}


The graphon filter approximation result is obtained by combining Prop. \ref{lemma:graphon-graph-filter-weighted} and Lemma \ref{lemma:transf-weighted-stochastic} through the triangle inequality.

\begin{proposition}[Graphon filter transferability to a stochastic graph] \label{thm:graphon-graph-filter-stochastic}
Let $Y=\Phi (X; \bbh, \bbW)$ be a graphon filter [cf. \eqref{eqn:spec-graphon_filter}] satisfying assumptions AS\ref{as2}--AS\ref{as3}. Given a stochastic graph $\bbG_{\tilde{n}}$ with GSO $\bbS_{\tilde{n}}$ [cf. Definition \ref{def:stochastic}], let $\bby_{\tilde{n}} = \Phi(\bbx_{\tilde{n}};\bbh,\bbS_{\tilde{n}}/n)$ be a graph filter instantiated from $\Phi (X; \bbh, \bbW)$ on this graph. Let $\chi_1, \chi_2, \chi_3 \in (0,0.3]$. For any $0 < c \leq 1$ and $n \geq 4/\chi_2$ satisfying assumption AS\ref{as4}, with probability at least $[1-2\chi_1]\times[1-\chi_2]\times[1-\chi_3]$ it holds that
\begin{align*}
\begin{split}
\|Y_{\tilde{n}}-Y\| \leq \bigg(A_h &+ \frac{\pi B_{\bbW_{\tilde{n}}}^c}{\delta^c_{\bbW\bbW_{\tilde{n}}}}\bigg)\frac{2(A_w\alpha (n,\chi_1)+\beta(n,\chi_3))}{n}\|X\| \\
&+ \frac{A_x\alpha (n,\chi_1)(a_h c+2)}{n} + 4a_h c\|X\|
\end{split}
\end{align*}
where $\bbW_{\tilde{n}} = \bbW_{\bbG_{\tilde{n}}}$ is the graphon induced by $\bbG_{\tilde{n}}$ and $Y_{\tilde{n}} = \Phi (X_{\tilde{n}}; \bbh, \bbW_{\tilde{n}})$ is the graphon filter induced by $\bby_{\tilde{n}} = \Phi (\bbx_{\tilde{n}}; \bbh, \bbS_{\tilde{n}}/n)$ \eqref{eqn:graphon_filter_ind}.
\end{proposition}
\begin{proof}
Prop. \ref{thm:graphon-graph-filter-stochastic} follows directly from Lemma \ref{lemma:transf-weighted-stochastic}, Prop. \ref{lemma:graphon-graph-filter-weighted} and the triangle inequality. Note that the probabilities $[1-\chi_1] \times [1-\chi_2]$ and $1-\chi_3$ are multiplied because Lemma \ref{lemma:transf-weighted-stochastic} holds for any weighted graph $\mbG_{\tilde{n}}$, i.e., it is independent of the weighted graph. 
\end{proof}

This approximation bound is similar to the approximation bound derived for weighted graphs in Prop. \ref{lemma:graphon-graph-filter-weighted}, with two important differences. The first is that, in addition to depending on the node stochasticity constant $\alpha(n,\chi_1)$, the transferability constant also depends on the edge stochasticity $\beta(n,\chi_3)$ [cf. Definition \ref{def:edge_stoc}] which accounts for the randomness of the edges of $\bbG_{\tilde{n}}$. Besides modifying the value of the transferability constant for stochastic graphs, note that $\beta(n, \chi_3)$ also lowers the confidence of the bound. The second difference is that the part of the bound due to non-transferable spectral components is twice that of Prop. \ref{lemma:graphon-graph-filter-weighted}. This is due to the summation of the non-transferable energy between the graphon and the weighted graph (Prop. \ref{lemma:graphon-graph-filter-weighted}), and between the weighted graph and the stochastic graph (Lemma \ref{lemma:transf-weighted-stochastic}).

Since a graphon identifies a family of graphs, the graphon filter approximation result in Thm. \ref{lemma:graphon-graph-filter} and Propositions \ref{lemma:graphon-graph-filter-template}--\ref{thm:graphon-graph-filter-stochastic} can be readily extended to a graph filter transferability result. To stick with the same example and avoid repetition, herein we only show transferability of graph filters across stochastic graphs, but note that this holds for any graph sampled from the graphon (including template and weighted) by Thm. \ref{lemma:graphon-graph-filter} and the tringle inequality. If we can bound the error made when we transfer (i) a graphon filter to a graph $\bbG_{\tilde{n}_1}$ and (ii) the same graphon filter to a graph $\bbG_{\tilde{n}_2}$, then we can bound the error of transferring a graph filter across $\bbG_{\tilde{n}_1}$ and $\bbG_{\tilde{n}_2}$ by the sum of (i) and (ii).

\begin{proposition}[Graph filter transferability on stochastic graphs]
\label{thm:graph-graph-filter-stochastic}
Let $Y=\Phi (X; \bbh, \bbW)$ be a graphon filter [cf. \eqref{eqn:spec-graphon_filter}] satisfying assumptions AS\ref{as2}--AS\ref{as3}. Let $\bbG_{\tilde{n}_1}$ and $\bbG_{\tilde{n}_2}$, $n_1 \neq n_2$, be two stochastic graphs with GSOs $\bbS_{\tilde{n}_1}$ and $\bbS_{\tilde{n}_2}$ respectively [cf. Definition \ref{def:stochastic}], and let $\bby_{\tilde{n}_1} = \Phi(\bbx_{\tilde{n}_1};\bbh,\bbS_{\tilde{n}_1}/n)$ and $\bby_{\tilde{n}_2} = \Phi(\bbx_{\tilde{n}_2};\bbh,\bbS_{\tilde{n}_2}/n)$ be graph filters instantiated from $\Phi (X; \bbh, \bbW)$ on these graphs.  Let $\chi_1, \chi_2, \chi_3 \in (0,0.3]$. For any $0 < c \leq 1$ and $n_1, n_2 \geq 4/\chi_2$ satisfying assumption AS\ref{as4}, with probability at least $[1-2\chi_1]^2\times[1-\chi_2]^2\times[1-\chi_3]^2$ it holds that
\begin{align*}
\begin{split}
\|&Y_{\tilde{n}_1}-Y_{\tilde{n}_2}\| \\
&\leq 4\bigg(A_h + \frac{\pi B^c_{\mbox{{\scriptsize max}}}}{\delta^c_{{\mbox{\scriptsize min}}}}\bigg)\max_{i \in \{1,2\}}\frac{A_w\alpha(n_{i},\chi_1) + {\beta(n_{i},\chi_3)}}{n_{i}}\|X\| \\
&+ 2A_x(a_h c+2)\max_{i \in \{1,2\}}\frac{\alpha(n_{i},\chi_1)}{n_{i}} + 8a_h c\|X\|
\end{split}
\end{align*}
where $B^c_{\mbox{{\scriptsize max}}} = \max_{i \in \{1,2\}}\max\big(B^c_{\mbW_{\tilde{n}_i}},B^c_{\bbW_{\tilde{n}_i}}\big)$ and $\delta^c_{{\mbox{\scriptsize min}}} = \min_{i \in \{1,2\}}\min\big(\delta^c_{\bbW\mbW_{\tilde{n}_i}},\delta^c_{\mbW_{\tilde{n}_i}\bbW_{\tilde{n}_i}}\big)$.
\end{proposition}
\begin{proof}
Prop. \ref{thm:graph-graph-filter-stochastic} follows from Prop. \ref{thm:graphon-graph-filter-stochastic} and the triangle inequality.
\end{proof}

Graph filters are thus transferable across graphs in the same graphon family. 
Since stochastic graphs have random edges,  Prop. \ref{thm:graph-graph-filter-stochastic} can be applied to graphs of same size. If $n_1\neq n_2$, the transferability bound is slightly loose because it is a simplification of the sum of the graphon filter approximation bounds for $\bbG_{\tilde{n}_1}$ and $\bbG_{\tilde{n}_2}$. This simplification is however useful as it shows that the transferability error of graph filters is dominated by the graph with the largest ratio of node/edge stochasticity to graph size, typically the smallest graph. 

Transferability of graph filters is an important result because it means that we can design a filter for one graph and transfer it to another. This is possible even if the graphs have different sizes, which considerably simplifies the linear processing of data supported on large-scale graphs. 
In the following, we will show that graph filter transferability also enables large-scale graph machine learning via GNNs. This is because the transferability properties of graph filters are not only inherited, but also augmented by GNNs.

\begin{remark}[Graph filter transferability on template and weighted graphs]
The graph filter transferability result for stochastic graphs in Prop \ref{thm:graph-graph-filter-stochastic} also holds for weighted and template graphs. Simply set $\beta(n, \chi_3) = 0$ for template and weighted graphs and $\alpha(n,\chi_1)=1$ for template graphs. Tighter bounds can be achieved by combining the triangle inequality with Propositions \ref{lemma:graphon-graph-filter-template} and \ref{lemma:graphon-graph-filter-weighted} respectively, in the same way that Prop \ref{thm:graphon-graph-filter-stochastic} was used to show Prop. \ref{thm:graph-graph-filter-stochastic}.
\end{remark}

%% file: figures/figures/1030/30_filter_small_eig.tex


%
\def \figurewidth   {17 cm}
\def \figureheight  { 4 cm}
\def \unit          { 1 cm}

%
\def \ymin { 0} \def \ymax {1}
\def \xmin {-0.4} \def \xmax {0.4}
\def \yaxispad {0.05 }
\def \xaxispad {0.05 }

\def \xfactor {18.89} 
\def \xoffset {8.5} 


{\fontsize{7}{7}\selectfont

\begin{tikzpicture} [ x = 1*\unit, y=1*\unit]

   {
   \path[fill=red, opacity = 0.05] 
              (-0.1*\xfactor + \xoffset, 0.05*\figureheight) --
              (-0.1*\xfactor + \xoffset, 0.95*\figureheight) --  
              ( 0.1*\xfactor + \xoffset, 0.95*\figureheight) -- 
              ( 0.1*\xfactor + \xoffset, 0.05*\figureheight) -- cycle;

   \path[fill=red, opacity = 0.2] 
              (-0.1*\xfactor + \xoffset, 0.38*\figureheight) --
              (-0.1*\xfactor + \xoffset, 0.64*\figureheight) --  
              ( 0.1*\xfactor + \xoffset, 0.64*\figureheight) -- 
              ( 0.1*\xfactor + \xoffset, 0.38*\figureheight) -- cycle;

   }

   %
   \begin{axis} [ scale only axis,
                  enlarge x limits = false,
                  width   = \figurewidth,
                  height  = \figureheight,
                  ymin        = \ymin - \yaxispad, 
                  ymax        = \ymax + \yaxispad,
                  xmin        = \xmin - \xaxispad, 
                  xmax        = \xmax + \xaxispad,
                  axis x line       = bottom,
                  axis x line shift = -\yaxispad,
                  axis line style   = {line width = 1.0, stealth-stealth},
                  axis y line       = none,
                  axis y line shift = -\xaxispad + \xmin,                  
                  xtick style = {draw=none},
                  xtick       = {\xmin, 0, \xmax, -0.1, 0.1},
                  xticklabels = {\xmin, 0, \xmax, -0.1, 0.1},
                  ytick style = {draw=none},
                  ytick       = {\ymin, \ymax},
                  yticklabels = { , },
                ]

   \def \spikesgraphon { +0.010, +0.020, +0.040, +0.070, +0.11, +0.16, +0.22, 
                         +0.290, +0.370, +0.460, +0.570, +0.69, +0.82, +0.96, 
                         -0.010, -0.030, -0.070, -0.130, -0.21, 
                         -0.310, -0.430, -0.570, -0.730, -0.91 }

   \def \spikesgraph {   +0.009, +0.018, +0.037, +0.066, +0.10, +0.14, +0.20, 
                         +0.270, +0.345, +0.435, +0.545, +0.66, +0.78, +0.91, 
                         -0.011, -0.023, -0.074, -0.135, -0.22, 
                         -0.330, -0.460, -0.610, -0.770, -0.96 }

   \def \frequencyresponse { ( 0.70*exp(-(1/0.3*(x+0.50))^2) +
                               0.50*exp(-(1/0.2*(x-0.10))^2) +
                               0.65*exp(-(1/0.2*(x-0.70))^2) + 0.1
                             ) 
                           }

   \addplot+[ domain=-1:1, 
              samples = 200,  
              mark         = node*, 
              solid, 
              line width   = 1.2,              
              color        =  red!90, 
            ]
            { \frequencyresponse };

   \addplot+[ samples at   = \spikesgraphon, 
              solid, 
              ycomb, 
              mark         = otimes*, 
              mark size    = 1.5pt,
              line width   = 0.8,              
              color        =  blue!20, 
              mark options = {blue!20}
            ]
            { 1 };
            
   \addplot+[ samples at   = \spikesgraphon, 
              solid, 
              ycomb, 
              mark         = otimes*, 
              mark size    = 1.5pt,
              line width   = 0.8,              
              color        =  blue!60, 
              mark options = {blue!60}
            ]
            { \frequencyresponse };

   \addplot+[ samples at   = \spikesgraph, 
              solid, 
              ycomb, 
              mark         = otimes*, 
              mark size    = 1.5pt,
              line width   = 0.8,              
              color        =  mygreen!20, 
              mark options = {mygreen!20}
            ]
            { 1  };         

   \addplot+[ samples at   = \spikesgraph, 
              solid, 
              ycomb, 
              mark         = otimes*, 
              mark size    = 1.5pt,
              line width   = 0.8,              
              color        =  mygreen!60, 
              mark options = {mygreen!60}
            ]
            { \frequencyresponse };         

   \addplot+[ samples at   = {-1, 0, 1}, 
              solid, 
              ycomb, 
              mark         = stealth, 
              line width   = 0.8,              
              color        =  black, 
              mark options = {blue!60}
            ]
            { 1.0 };

   \addplot+[ domain       = -1:1,
              samples      = 2,
              solid, 
              mark         = none, 
              line width   = 1.0,              
              color        =  black, 
              mark options = {blue!60}
            ]
            { 1 };

\end{axis}

\end{tikzpicture}

}

%% file: gnn-transferability.tex



 
In the following, we discuss the transferability properties of GNNs. We begin by defining $\bbW$NNs in Sec. \ref{sbs:WNNs}. Then, in Sec. \ref{sbs:gnn-transf} we interpret $\bbW$NNs as generative models for GNNs, and prove Thm. \ref{thm:graphon-graph-nn}, which derives an upper bound for the transferability or approximation error between a $\bbW$NN and a GNN sampled from it on a generic graph sampled from the graphon.
The particularization of this theorem to stochastic graphs is stated in Prop. \ref{thm:graphon-graph-nn-stochastic}; finally, in Prop. \ref{thm:graph-graph-nn-stochastic}, we show how Prop. \ref{thm:graphon-graph-nn-stochastic} implies transferability of GNNs across stochastic graphs sampled from the same graphon.


\subsection{Graphon Neural Networks} \label{sbs:WNNs}

A graphon neural network ($\bbW$NN) is a deep convolutional architecture consisting of layers where each layer implements a convolutional filterbank followed by a pointwise nonlinearity \cite{ruiz20-transf}.
Consider layer $\ell$, which maps the incoming $F_{\ell-1}$ features from layer $\ell-1$ into $F_\ell$ features.
The first step in this layer is to process the features $X_{\ell-1}^g$, $1 \leq g \leq F_{\ell-1}$, with a convolutional filterbank to generate the $F_\ell$ intermediate linear features $U_{\ell}^f$,
\begin{equation} \label{eqn:wnn-int}
U^f_{\ell} = \sum_{g=1}^{F_{\ell-1}} T_{\bbH^{fg}_\ell} X^g_{\ell-1}
\end{equation}
where $1 \leq f \leq F_\ell$. Each intermediate feature $U_\ell^f$ is obtained by aggregating the outputs of $F_{\ell-1}$ filters like the one in \eqref{eqn:lsi-wf} with coefficients $\bbh^{fg}_\ell$.
Since there are $F_\ell$ such intermediate features, the filterbanks at each layer of the $\bbW$NN contain a total of $F_\ell \times F_{\ell-1}$ convolutional filters. 

The next step is to process the intermediate features $U_\ell^f$ with a pointwise nonlinearity, e.g., the ReLU. Denoting this nonlinearity $\sigma$, the $f$th feature of the $\ell$th layer is given by 
\begin{equation} \label{eqn:wnn-lay}
X^f_{\ell} = \sigma\left(U_\ell^f\right)
\end{equation}
for $1 \leq f \leq F_\ell$. Because the nonlinearity is pointwise, $X^f_\ell(u) = \sigma(U_\ell^f(u))$ for all $u \in [0,1]$. If the $\bbW$NN has $L$ layers, \eqref{eqn:wnn-int}--\eqref{eqn:wnn-lay} are repeated $L$ times.
The input features at the first layer, $X_0^g$, are the input data $X^g$ for $1 \leq g \leq F_0$, and the $\bbW$NN output is given by $Y^f = X_L^f$ for $1 \leq f \leq F_L$. 

Similarly to the graphon filter, a $\bbW$NN with inputs $X=\{X^g\}_{g=1}^{F_0}$ and outputs $Y = \{Y^f\}_{f=1}^{F_L}$ can be represented more compactly as the map $Y = \Phi(X; \ccalH, \bbW)$, where $\ccalH$ is a tensor grouping the coefficients $\bbh_\ell^{fg}$ for all features and all layers of the $\bbW$NN, i.e., $\ccalH = [\bbh_\ell^{fg}]_{\ell f g}$ for $1 \leq \ell \leq L$, $1 \leq f \leq F_{\ell}$ and $1 \leq g \leq F_{\ell-1}$. Comparing this map with the GNN map in \eqref{eqn:gcn_map}, we see that, except for the fact that their supports---a graphon and a graph respectively---are different, if the tensors $\ccalH$ are equal, these maps are the same. This allows interpreting $\bbW$NNs as generative models for GNNs where the graph $\bbG_n$ is instantiated from $\bbW$ as in Definitions \ref{def:template}, \ref{def:weighted} or \ref{def:stochastic}, and the graph signal $\bbx_n$ is instantiated from $X$ as in \eqref{eqn:template_graph_signal} or \eqref{eqn:weighted_graph_signal}. 

The interpretation of $\bbW$NNs as generative models for GNNs is important for two reasons. First, it allows designing one $\bbW$NN and instantiating as many GNNs as desired from it. I.e, it allows designing neural networks in the limit of very large graphs and transferring them to finite graphs without changes to the architecture. Second, it motivates analyzing the ability to transfer GNNs across graphs of same or different size, since a sequence of graphs instantiated from a graphon following any of Definitions \ref{def:template}, \ref{def:weighted} or \ref{def:stochastic} converges in probability \cite[Chapter 11]{lovasz2012large}.
 

To be able to compare $\bbW$NNs with GNNs, or two GNNs supported on graphs of different sizes, we will need to define $\bbW$NNs induced by GNNs. The $\bbW$NN induced by a GNN $\bby_n = \Phi(\bbx_n; \ccalH, \bbS_n/n)$ on the graph $\bbG_n$ is given by 
\begin{equation} \label{eqn:wnn_ind}
Y_n = \Phi(X_n; \ccalH, \bbW_n) 
\end{equation}
where the graphon $\bbW_n = \bbW_{\bbG_n}$ is the graphon induced by $\bbG_n$ \eqref{eqn:graphon_ind} and $X_n$ is the graphon signal induced by the graph signal $\bbx_n$ \eqref{eqn:graphon_signal_ind}. Note the normalization by $n$, which is necessary to have $Y_n(u) = \sum_{i=1}^n [\bby_n]\mbI(u \in I_i)$, see Lemma \ref{lemma:induced_filters}.


\subsection{Graph Neural Network Transferability} \label{sbs:gnn-transf}

Consider a $\bbW$NN with $L$ layers, $F_0=1$ input feature, $F_L = 1$ output feature, and $F_\ell = F$ features per layer for $1 \leq \ell \leq L-1$. 
Under a Lipschitz continuity assumption on the nonlinearity $\sigma$, this $\bbW$NN can be approximated by a GNN on a graph $\bbG_n$ sampled from the graphon $\bbW$.

\begin{assumption} \label{as5}
The activation functions are normalized Lipschitz, i.e., $|\sigma(x)-\sigma(y)| \leq |x-y|$, and $\sigma(0)=0$.
\end{assumption} 

\begin{theorem}[$\bbW$NN transferability to a generic graph] \label{thm:graphon-graph-nn}
Let $Y=\Phi (X; \ccalH, \bbW)$ be a $\bbW$NN with $L$ layers, $F_0=F_L=1$ input and output features and $F_\ell = F$, $1 \leq \ell < L$ [cf. \eqref{eqn:wnn-int}--\eqref{eqn:wnn-lay}]. Assume that the convolutional filters that make up its layers all satisfy assumption AS\ref{as2}. Given a generic graph $\bbG_n$ with GSO $\bbS_n$ sampled from $\bbW$, let $\bby_n = \Phi(\bbx_n;\ccalH,\bbS_n/n)$ be a GNN instantiated from $\Phi (X; \ccalH, \bbW)$ on this graph. For any $0 < c \leq 1$, it holds that
\begin{align} \label{eqn:thm-graphon-graph-nn}
\begin{split}
\|Y_{n}-Y\| &\leq LF^{L-1}\bigg(A_h + \frac{\pi B_{\bbW_n}^c}{\delta^c_{\bbW\bbW_n}}\bigg)\|\bbW-\bbW_n\|\|X\| \\
&+ (A_h c + 2)\|X-X_n\| + LF^{L-1}2a_h c\|X\|
\end{split}
\end{align}
where $\bbW_n := \bbW_{\bbG_n}$ is the graphon induced by $\bbG_n$ and $Y_{n} = \Phi (X_{n}; \ccalH, \bbW_n)$ is the $\bbW$NN induced by $\bby_n = \Phi (\bbx_n; \ccalH, \bbS_n/n)$ \eqref{eqn:wnn_ind}.
\end{theorem}
\begin{proof}
Refer to Appendix \ref{sec:appendixC} in the supplementary material.
\end{proof}

This theorem follows recursively from Thm. \ref{lemma:graphon-graph-filter} because, provided that Assumption AS\ref{as5} is met, $|\sigma(U_\ell + \Delta U_\ell) - \sigma (U_\ell)|$ can be upper bounded by $|\Delta U_\ell|$. Most common activation functions, e.g., the hyperbolic tangent and the ReLU, satisfy this assumption. Thus, we can bound the error incurred of approximating $\bbW$NN with a GNN by the transferability error of a cascade of $L$ graphon filterbanks.

Akin to the graphon filter approximation bound, the approximation bound in Thm. \ref{thm:graphon-graph-nn} has three terms: the first, which is controlled by the transferability constant and the distance between the graph and the graphon; the second, which is controlled by the distance between the graph signal and the graphon signal; and the third, which corrresponds to the non-transferable energy of the input signal. 
Because it arises from the sampling error of the input signal, the second error term is exactly the same as in Thm. \ref{lemma:graphon-graph-filter}, while the others differ by a scaling factor of $LF^{L-1}$. Hence, the deeper and the wider the $\bbW$NN, the more difficult it is to approximate it with a GNN.

In Prop. \ref{thm:graphon-graph-nn-stochastic}, we give a concrete example of how a $\bbW$NN may be approximated by a GNN on a graph sampled from the graphon. For brevity, we only particularize the approximation result of Thm. \ref{thm:graphon-graph-nn} to the more general case of the stochastic graphs in Definition \ref{def:stochastic}, but it is easy to write the same result for template and weighted graphs (see Remark \ref{remark2}).

\begin{proposition}[$\bbW$NN transferability to a stochastic graph] \label{thm:graphon-graph-nn-stochastic}
Let $Y=\Phi (X; \ccalH, \bbW)$ be a $\bbW$NN with $L$ layers, $F_0=F_L=1$ input and output features and $F_\ell = F$, $1 \leq \ell < L$ [cf. \eqref{eqn:wnn-int}--\eqref{eqn:wnn-lay}]. Assume that this $\bbW$NN satisfies assumptions AS\ref{as1}, AS\ref{as3}, and AS\ref{as5}, and that the convolutional filters that make up its layers all satisfy assumption AS\ref{as2}. Given a stochastic graph $\bbG_{\tilde{n}}$ with GSO $\bbS_{\tilde{n}}$ [cf. Definition \ref{def:stochastic}], let $\bby_{\tilde{n}} = \Phi(\bbx_{\tilde{n}};\ccalH,\bbS_{\tilde{n}}/n)$ be a GNN instantiated from $\Phi (X; \ccalH, \bbW)$ on this graph. Let $\chi_1, \chi_2, \chi_3 \in (0,0.3]$. For any $0 < c \leq 1$ and $n \geq 4/\chi_2$  satisfying assumption AS\ref{as4}, with probability at least $[1-2\chi_1]\times[1-\chi_2]\times[1-\chi_3]$ it holds that
\begin{align} \label{eq:thm3}
\begin{split}
\|Y_{\tilde{n}}&-Y\| \leq LF^{L-1}\bigg(A_h + \frac{\pi B_{\bbW_{\tilde{n}}}^c}{\delta^c_{\bbW\bbW_{\tilde{n}}}}\bigg) \\
&\times \frac{2(A_w\alpha (n,\chi_1)+\beta(n,\chi_3))}{n}\|X\| \\
&+ \frac{A_x\alpha (n,\chi_1)(a_h c+2)}{n} + LF^{L-1}4a_h c\|X\|
\end{split}
\end{align}
where $\bbW_{\tilde{n}} = \bbW_{\bbG_{\tilde{n}}}$ is the graphon induced by $\bbG_{\tilde{n}}$ and $Y_{\tilde{n}} = \Phi (X_{\tilde{n}}; \ccalH, \bbW_{\tilde{n}})$ is the $\bbW$NN induced by $\bby_{\tilde{n}} = \Phi (\bbx_{\tilde{n}}; \ccalH, \bbS_{\tilde{n}}/n)$ \eqref{eqn:wnn_ind}.
\end{proposition}
\begin{proof} \renewcommand{\qedsymbol}{}
The proof follows the same steps of the proof of Thm. \ref{thm:graphon-graph-nn} in Appendix \ref{sec:appendixC} of the supplementary material, using Prop. \ref{thm:graphon-graph-filter-stochastic} instead of Thm. \ref{lemma:graphon-graph-filter} in \eqref{eqn:step_change}.
\end{proof} 

From Prop. \ref{thm:graphon-graph-nn-stochastic}, it is ready to show that GNNs are transferable across stochastic graphs.

\begin{proposition}[GNN transferability across stochastic graphs]
\label{thm:graph-graph-nn-stochastic}
Let $Y=\Phi (X; \ccalH, \bbW)$ be a $\bbW$NN [cf. \eqref{eqn:gcn_map}] satisfying assumptions AS\ref{as1}, AS\ref{as3}, AS\ref{as5}, and such that the convolutional filters at all layers satisfy assumption AS\ref{as2}. Let $\bbG_{\tilde{n}_1}$ and $\bbG_{\tilde{n}_2}$, $n_1 \neq n_2$, be two stochastic graphs with GSOs $\bbS_{\tilde{n}_1}$ and $\bbS_{\tilde{n}_2}$ respectively [cf. Definition \ref{def:stochastic}], and let $\bby_{\tilde{n}_1} = \Phi(\bbx_{\tilde{n}_1};\ccalH,\bbS_{\tilde{n}_1}/n)$ and $\bby_{\tilde{n}_2} = \Phi(\bbx_{\tilde{n}_2};\ccalH,\bbS_{\tilde{n}_2}/n)$ be GNNs instantiated from $\Phi (X; \ccalH, \bbW)$ on these graphs.  Let $\chi_1, \chi_2, \chi_3 \in (0,0.3]$. For any $0 < c \leq 1$ and $n_1, n_2 \geq 4/\chi_2$ satisfying assumption AS\ref{as4}, with probability at least $[1-2\chi_1]^2\times[1-\chi_2]^2\times[1-\chi_3]^2$ it holds that
\begin{align*}
\begin{split}
\|Y_{\tilde{n}_1}&-Y_{\tilde{n}_2}\|\leq 4LF^{L-1}\bigg(A_h + \frac{\pi B^c_{\mbox{{\scriptsize max}}}}{\delta^c_{{\mbox{\scriptsize min}}}}\bigg)\\
&\times \max_{i \in \{1,2\}}\frac{A_w\alpha(n_{i},\chi_1) + {\beta(n_{i},\chi_3)}}{n_{i}}\|X\|\\
&+2A_x(a_h c+2)\max_{i \in \{1,2\}}\frac{\alpha(n_{i},\chi_1)}{n_{i}} + 8LF^{L-1}a_h c\|X\|
\end{split}
\end{align*}
where $B^c_{\mbox{{\scriptsize max}}} = \max_{i \in \{1,2\}}\max\big(B^c_{\mbW_{\tilde{n}_i}},B^c_{\bbW_{\tilde{n}_i}}\big)$ and $\delta^c_{{\mbox{\scriptsize min}}} = \min_{i \in \{1,2\}}\min\big(\delta^c_{\bbW\mbW_{\tilde{n}_i}},\delta^c_{\mbW_{\tilde{n}_i}\bbW_{\tilde{n}_i}}\big)$.
\end{proposition}
\begin{proof}
Prop. \ref{thm:graph-graph-nn-stochastic} follows directly from Prop. \ref{thm:graphon-graph-nn-stochastic} and the triangle inequality.
\end{proof}

When $n_1=n_2$, the GNN transferability bound in Prop. \ref{thm:graph-graph-nn-stochastic} is approximately twice the $\bbW$NN approximation bound in Prop. \ref{thm:graphon-graph-nn-stochastic}. When $n_1 \neq n_2$, this bound can be improved by explicitly writing the sum of \eqref{eq:thm3} for $\bbG_{\tilde{n}_1}$ and  $\bbG_{\tilde{n}_2}$; we take the maximum to highlight that the error is dominated by the inverse of the size of smallest graph. Although Prop. \ref{thm:graph-graph-nn-stochastic} is explicitly written for stochastic graphs, it also holds for weighted and template graphs associated with the same graphon; see Remark \ref{remark2}.

The main implication of Prop. \ref{thm:graph-graph-nn-stochastic} is that a GNN can be trained on one graph to be transferred to another graph, which is especially helpful when the graph on which we want to execute the GNN is large and training on it is costly. In these cases, we can use the bound in Prop. \ref{thm:graph-graph-nn-stochastic} to determine the minimum size of the graph on which the GNN should be trained to meet a given error allowance. The transferability property of the GNN thus makes it a suitable model for  machine learning on large-scale graphs.

\begin{remark} \label{remark2}
The GNN transferability result derived for stochastic graphs in Prop. \ref{thm:graph-graph-nn-stochastic} holds for weighted graphs by setting $\beta(n, \chi_3) = 0$ for weighted and template graphs and $\alpha(n,\chi_1)=1$ for the latter. However, the resulting bounds are slightly loose. For tighter bounds, combine the triangle inequality with Props. \ref{lemma:graphon-graph-filter-template} and \ref{lemma:graphon-graph-filter-weighted} respectively to obtain the graph filter transferability bounds for template and stochastic graphs. Then, follow the same proof steps used to show Thm. \ref{thm:graphon-graph-nn} from Thm. \ref{lemma:graphon-graph-filter} in Appendix \ref{sec:appendixC} of the supplementary material. The graph filter transferability result derived for stochastic graphs in Prop. \ref{thm:graph-graph-filter-stochastic} also holds for weighted and template graphs. Simply set $\beta(n, \chi_3) = 0$ for both types of graphs and $\alpha(n,\chi_1)=1$ for template graphs. Tighter bounds can be achieved by combining the triangle inequality with Props. \ref{lemma:graphon-graph-filter-template} and \ref{lemma:graphon-graph-filter-weighted} respectively, in the same way that Prop. \ref{thm:graphon-graph-filter-stochastic} was used to show Prop. \ref{thm:graph-graph-filter-stochastic}.
\end{remark}

\subsection{Discussion} \label{sbs:discussion}

\myparagraph{Non-transferable energy, filters with constant band, and asymptotics.}
In order to be transferable, filters and GNNs have to be able to ``match'' the eigenvalues of the source graph with those of the target graph so that the amplifications of the output signal's spectral components match on both graphs. This is possible because, as illustrated in Fig. \ref{fig:eigenvalues_and_filter}, as $n \to \infty$ the eigenvalues of a graph $\bbG_n$ sampled from a graphon $\bbW$ converge to the graphon eigenvalues \cite[Chapter 11.6]{lovasz2012large}. Hence, for large enough $\bbG_{n_1}$ and $\bbG_{n_2}$, their eigenvalues are close. However, because the graphon eigenvalues accumulate near zero, for small eigenvalues (i.e., for $\lambda_j$ such that $|j| \to \infty$), this matching becomes very hard. This is due to the fact that the distance between the eigenvalues $\lambda_j(\bbG_{n_1})$ and $\lambda_j(\bbG_{n_2})$ might be larger than the distance between consecutive eigenvalues in this range. 
To mitigate this problem, as shown in Fig. \ref{fig:eigenvalues_and_filter} we restrict the variability of the filters to $a_h < A_h$ below a certain threshold $c$ [cf. Assumption AS\ref{as4}]. This ensures that the amplifications of the spectral components below $c$ won't be too different, and the fact that they cannot be discriminated less problematic. Still, because $a_h$ is nonzero, the transference of these spectral components will incur in an error---the third term of the transferability bound in Thms. \ref{lemma:graphon-graph-filter} and \ref{thm:graphon-graph-nn}. We refer to the energy of this error as the \textit{non-transferable energy}, since it stems from spectral components which cannot be transferred from graph to graph. 

Due to $a_h$ being small, the non-transferable energy does not contribute much to the transferability error and is dominated by the part of the bound that depends on the transferability constant. However, this is largely dependent on the value of $c$. Reducing the value of $c$ reduces the contribution of the non-transferable energy to the transferability bound. But decreasing $c$ also has the effect of increasing the transferability constant both through $B^c_{\mbox{\tiny max}}$, because a lower value of $c$ results in a larger number of eigenvalues in $[-1,-c] \cup [c,1]$; and through $\delta^c_{\mbox{\tiny min}}$, because as $c$ approaches zero so does the margin between consecutive eigenvalues, which is the limit of $\delta^c_{\mbox{\tiny min}}$ as $n \to \infty$. 

The non-transferable energy is so called because, for fixed $c$, it is a constant term in the transferability bound which does not decrease with the size of the graphs. To avoid non-transferable spectral components, the graph filter (or, respectively, the graph filters of the GNN) would need to have a constant frequency response for $|\lambda| < c$ or, equivalently, $a_h =0$. However, such filters are undesirable in practice because since they are not analytic they cannot be written in convolutional form \eqref{eqn:graph_convolution}.  
Nonetheless, in the limit it is possible to show that the difference between the outputs of Lipschitz continuous graph convolutions supported on $\bbG_{n_1}$ and $\bbG_{n_2}$ converging to the same graphon (and therefore of GNNs constructed with such convolutions) vanishes as $n_1,n_2\to\infty$. This convergence result, which can be seen as an asymptotic transferability result, is proved in \cite[Thm. 4]{ruiz2020graphonsp}.

\myparagraph{Transferability-discriminability tradeoff and the effect of nonlinearities.} In both the graph filter transferability result (Thm. \ref{lemma:graphon-graph-filter}) and its GNN counterpart (Thm. \ref{thm:graphon-graph-nn}), the transferability constant (i.e., the first term of the bound) depends on the parameters $B^c_{\mbox{\tiny max}}$ and $\delta^c_{\mbox{\tiny min}}$ of the graph convolutional filters, which in turn depend on the value of $c$. The parameter $B^c_{\mbox{\tiny max}}$ is the maximum $c$-band cardinality of a graphon [cf. Definition \ref{def:c_band_card}], which counts the number of graphon eigenvalues larger than $c$. The parameter $\delta^c_{\mbox{\tiny min}}$ is the minimum $c$-eigenvalue margin between two graphons [cf. Definition \ref{def:c_eig_margin}], which measures the minimum distance between eigenvalues of these two graphons with consecutive indices where one is smaller and the other is larger than $c$. 
Since the eigenvalues of the limit graphon accumulate near zero, if $c$ is large (i.e., close to one), $B^c_{\mbox{\tiny max}}$ is small---because there are less eigenvalues in the $[-1,-c] \cup [c,1]$ interval---and $\delta^c_{\mbox{\tiny min}}$ is large---because the further away eigenvalues are from zero, the larger the distance between two consecutive eigenvalues. This leads to a smaller transferability constant, decreasing the transferability error bound. On the other hand, larger values of $c$ also reduce the model's discriminative power (or discriminability) because they decrease the length of the interval $[-1,-c]\cup[c,1]$ where the filter has full variability (i.e., $A_h$). Hence, the bounds in Thms. \ref{lemma:graphon-graph-filter} and \ref{thm:graphon-graph-nn} exhibit a tradeoff between transferability and discriminability for both graph filters and GNNs. The value of the Lipschitz constant $A_h$ also plays a role in this trade-off, as higher values of $A_h$ lead to more discriminability but increase the transferability bound. Note that while for finite graphs this tradeoff arises from an analysis of the error bounds in Thms. \ref{lemma:graphon-graph-filter} and \ref{thm:graphon-graph-nn} (and not of the transferability error per se, whose specific form we do not know), this tradeoff is not an artifact of the specific form of the bound, as it also arises in the asymptotic convergence result from \cite{ruiz2020graphonsp}.

In the case of GNNs, we conjecture that the transferability-discriminability tradeoff is better than in graph filters because of the addition of nonlinearities. This is because nonlinearities have an effect akin to rectifiers, effectively \textit{scattering} some of the spectral components associated with small eigenvalues to the middle range of the spectrum where they can then be discriminated by the filters in the next layer. Indeed, we will observe this empirically in the experiments of Sec. \ref{sec:sims}: for the same level of discriminability (i.e., comparable performances on the graphs on which they are trained), GNNs exhibit a lower transferability error than graph filters. 
It is hypothesized that nonlinearities play a similar role in the stability of GNNs, but in that case, it is the scattering of the components associated with large eigenvalues that improves stability \cite{gama2019stability}.

\myparagraph{Graph as a design parameter.} While to prove transferability we focus on the limit object interpretation of $\bbW$NNs, their interpretation as generative models is also valuable because it allows looking at the graph as a tunable parameter of the GNN. I.e., instead of considering the graph to be a fixed hyperparameter, we could think of it it as a learnable parameter of the architecture like the weights $\ccalH$. The interpretation of the graph as a design parameter motivates a number of interesting research directions, e.g., building more general GNN architectures (i.e., with larger degrees of freedom); designing adversarial graph perturbations for GNNs; and drawing deeper connections between GNNs and transformers \cite{vaswani2017attention}, which can be thought of as GNN architectures where the graph is learned.

%% file: sims-transferability.tex


We illustrate the transferability properties of graph filters and GNNs in two applications: movie recommendation on a movie similarity network, and flocking via decentralized robot control. All architectures are trained using ADAM with learning rate $5\times 10^{-4}$ and forgetting factors $0.9$ and $0.999$.

\begin{figure*}[t]
\centering
\begin{subfigure}{0.3\textwidth} 
\centering
\includegraphics[height=4.5cm]{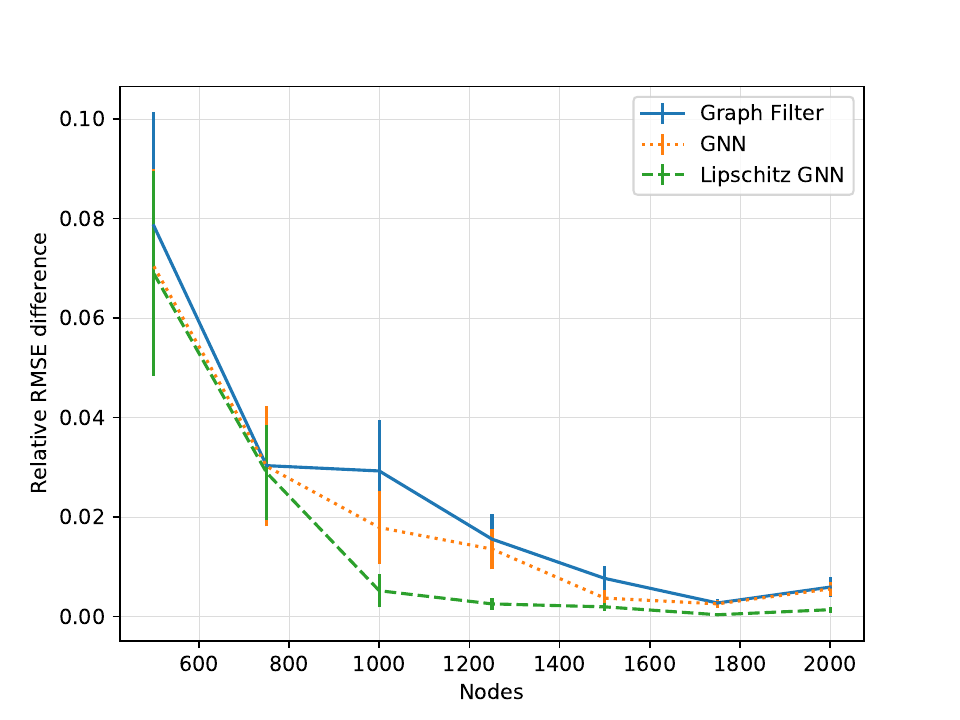}
\caption{}
\label{fig:movie}
\end{subfigure}
\begin{subfigure}{0.3\textwidth}
\centering
\includegraphics[height=4.5cm]{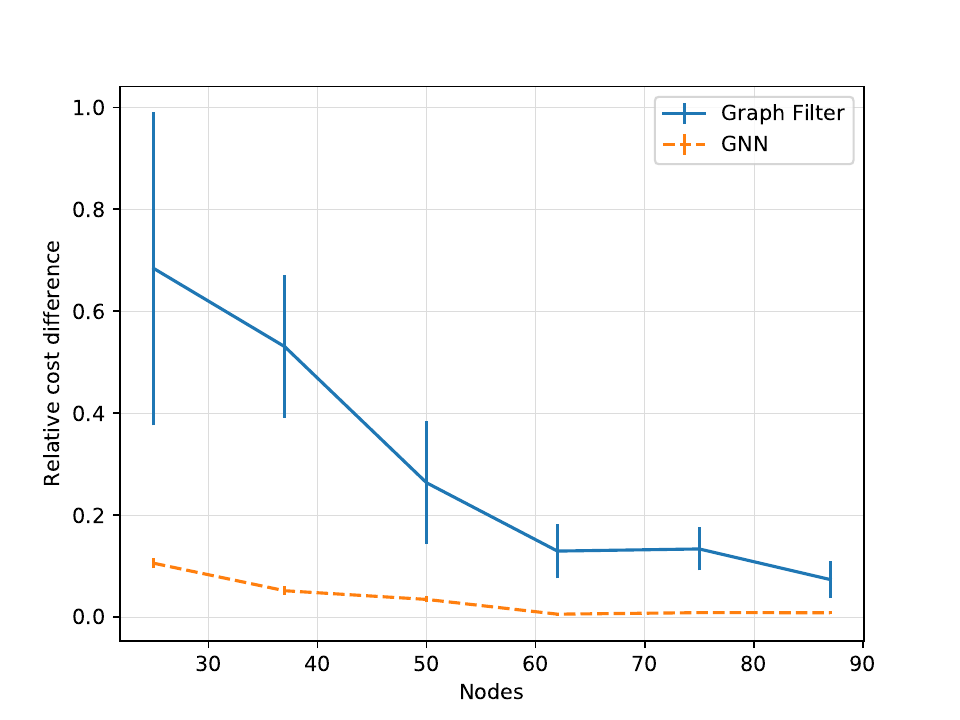}
\caption{}
\label{fig:flocking1}
\end{subfigure}
\begin{subfigure}{0.3\textwidth}
\centering
\includegraphics[height=4.5cm]{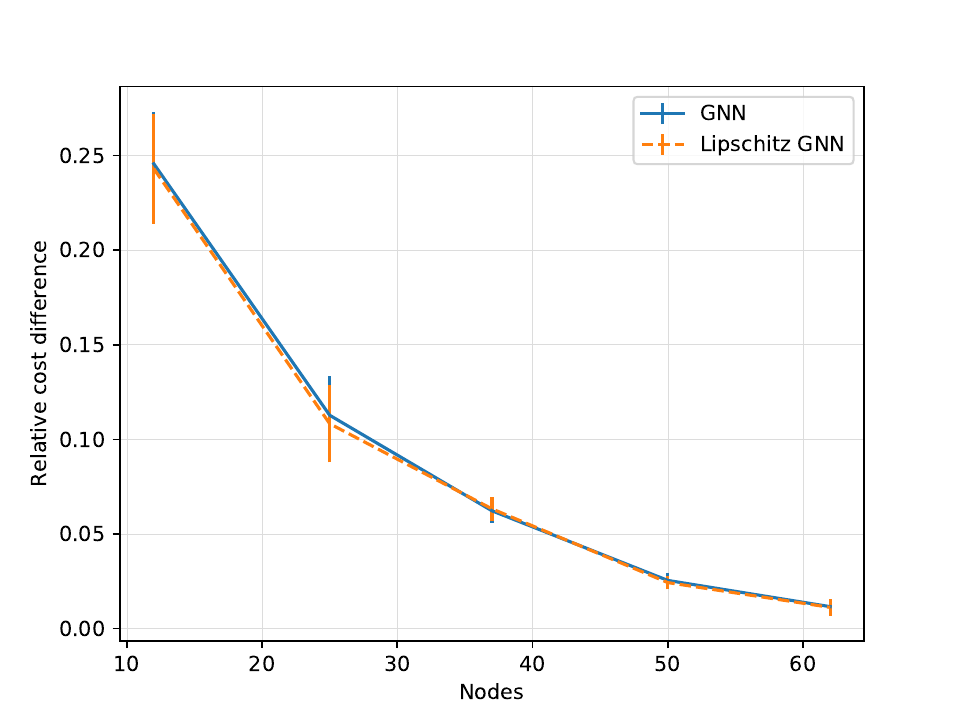}
\caption{}
\label{fig:flocking2}
\end{subfigure}
\caption{\subref{fig:movie} Difference between the RMSEs achieved on the training network and on the full $3416$-movie network for the graph filter, the GNN, and the GNN trained by penalizing convolutions with high Lipschitz constant.
The errors achieved by training on the full movie network are $0.81\pm 0.04$ for the graph filter and $0.81\pm 0.04$ for the GNN.
\subref{fig:flocking1} Difference between the control cost achieved on the training network and on the full $100$-agent network, relative to the cost on the training network, for a graph filter and a GNN. 
\subref{fig:flocking2} Difference between the control cost achieved on the training network and on the full $100$-agent network, relative to the cost on the original network, for a GNN and a GNN trained by penalizing convolutions with high Lipschitz constant. 
The costs achieved by training on the $100$ agent network are $19.31\pm 20.66$ for the graph filter and $1.46 \pm 0.01$ for the GNN. The error bars in \subref{fig:movie}--\subref{fig:flocking2} are scaled by $0.5$.}
\label{fig:experiments}
\end{figure*}

\subsection{Movie Recommendation}

We consider the MovieLens-1M dataset \cite{harper16-movielens}, which consists of one million ratings given by $6000$ users to $4000$ movies. Each rating is a score between 1 and 5 and higher scores indicate higher preference. Specifically, we aim to predict the ratings given by different users to the movie ``Star Wars: Episode IV - A New Hope'', which has a total of $2991$ ratings. To do so, we define a movie similarity network, where each node is a movie and each edge is the similarity between two movies. We restrict attention to movies with at least $5$ ratings, which brings the number of nodes of the full movie similarity network down to $3416$.  We compute this network by defining a training set with 90\% of the users, and by defining the similarity between two movies as the correlation between the ratings given by users in the training set to these movies. 
Each user corresponds to a graph signal. At each node, the value of a signal is the rating given by the user to the corresponding movie, or zero if a rating is not available. To train graph filters and GNNs in a supervised manner, we constructed input-output pairs where, for each user, the output is the the rating to `Star Wars'' and the input is the user's graph signal with the rating to this movie zeroed out.

The architectures we consider are a $1$-layer graph filter and a $1$-layer GNN. Both have $F_0=1$ input features, $F_1=64$ features in the first layer, and $K=5$ filter taps. The nonlinearity of the GNN is the ReLU and both the graph filter and the GNN are followed by a readout layer which outputs $1$ feature, the rating.
These architectures are trained by minimizing the MSE at the node corresponding to the movie ``Star Wars''. We set aside 10\% of the training samples for validation and train over $30$ epochs with batch size $16$. Performance is measured by recording the root mean squared error (RMSE) achieved by each architecture on the test set. To analyze transferability, we train the graph filter and the GNN on subnetworks with $n=500,750,1000,1250,1500,1750,2000$ movies, and plot the difference between the RMSE they achieve on the subnetwork of size $n$ and on the full $3416$-movie network, relative to the former, in Figure \ref{fig:movie}.

In Figure \ref{fig:movie}, we observe that the RMSE difference decreases with the size of the movie subnetwork for both the graph filter and the GNN as predicted by Thms. \ref{lemma:graphon-graph-filter} and \ref{thm:graphon-graph-nn}.
Although the GNN (orange) achieves a lower transferability error than the graph filter (blue) for most values of $n$, the difference is very small. We hypothesize that this is because the graph convolutions being learned in the GNN have high variability. To test this hypothesis, we retrain the GNN by minimizing the penalized MSE loss with penalization factor given by the maximum Lispchitz constant of the filters of the GNN scaled by a penalty multiplier. The transferability error of this GNN is shown in green. The Lipschitz GNN is more transferable than the graph filter and the GNN trained without regularization, corroborating that GNNs are more transferable than graph filters, and  illustrating the transferability-discriminability tradeoff of Thm. \ref{thm:graphon-graph-nn}.

\subsection{Flocking via Decentralized Robot Control}

Decentralized control problems consist of a team of $n$ agents which must accomplish a shared goal. 
Each agent has access to local states $\bbx_i$ and generates local control actions $\bba_i$. In order to learn which actions to perform, agents exchange information across pairwise communication links determined by their geographical proximity, which defines an agent proximity network $\bbG_n$. Because communication incurs in delays, if agents $i$ and $j$ are $k$ hops away from one another in $\bbG_n$, at time $t$ $i$ only has access to the delayed state $\bbx_j(t-k)$. We define the information history of agent $i$ as \cite{gama2021graph}
\begin{equation} \label{eqn:partialInformation}
    \ccalX_{i}(t) 
       = \bigcup_{k=0}^{K-1} 
          \Big\{ 
             \bbx_{j}(t-k) 
                : j \in \ccalN_{i}^{k}(t) 
                   \Big\} \text{.}
\end{equation}
This emphasizes that at time $t$ agent $i$ only knows its current state $\bbx_i(t)$ and the states of its $k$-hop neighbors $j \in \ccalN_i^k$ at time $t-k$. 

The information history in \eqref{eqn:partialInformation} allows defining a decentralized control scheme in which the actions $\bba_i(t)$ can be calculated as functions of the history $\ccalX_{i}(t)$. We can then use graph filters and GNNs to parametrize these functions by incorporating the delayed information structure into \eqref{eqn:graph_convolution}. Explicitly, to account for communication delays we rewrite the terms $\bbS^k\bbx$ in \eqref{eqn:graph_convolution} as $\prod_{\kappa=1}^{k}\bbS(t-\kappa)\bbx(t-k)$.
The convolutional filters in the GNN \eqref{eqn:gcn_layer1}--\eqref{eqn:gcn_layer2} are modified in the same way.

In the flocking problem \cite{tolstaya2020learning}, the shared goal is for all the agents to move with the same velocity while avoiding collisions. The states $\bbx_{i}(t) \in \reals^{6}$ are given by \cite{gama2021graph}
\begin{align} \label{eqn:flockingState}
   \bbx^T_{i}(t) 
      = \Bigg[ 
           \sum_{j \in \ccalN_{i}} 
              \bbv_{ij}(t) ;
                 \sum_{j \in \ccalN_{i}(t)} 
                    \frac{\bbr_{ij}}
                       {\| \bbr_{ij}(t)\|^{4}} ; 
                          \sum_{j \in \ccalN_{i}} 
                             \frac{\bbr_{ij}(t)}
                                {\| \bbr_{ij}(t)\|^{2}}
                                   \Bigg]
\end{align}
where $\bbr_{ij}(t)$ are the positions and $\bbv_{ij}(t)$ the velocities of agent $j$ measured relative to the positions and velocities of agent $i$ respectively. The neighborhood $\ccalN_i$ consists of nodes $j$ such that $\|\bbr_{ij}\|\leq R$, i.e., which are within a communication and sensing radius of length $R$ from $i$. We consider $R=2$. At $t=0$, the agents' positions and velocities are initialized at random. The actions $\bba_i(t)$ to predict are the agents' accelerations.

While a centralized solution to the flocking problem is straightforward---it suffices to direct all the agents to move in the same direction---the optimal decentralized controller is unknown. Hence, we train decentralized graph filters and GNNs following the information structure in \eqref{eqn:partialInformation} to ``imitate'' the centralized controller. We consider 2-layer graph filters and a GNNs with $F_0=6$ input features, $F_1=64$ features in the first layer, and $F_2=32$ features in the second layer. At all layers, the convolutional filters have $K=3$ filter taps. The readout layer outputs 2 features per node corresponding to the agents' accelerations in the $x$ and $y$ directions. The nonlinearity used in the GNN is the hyperbolic tangent.

Both architectures are trained by minimizing the mean squared error (MSE) over $400$ training trajectories of duration equal to $100$ steps. We train for $30$ epochs using minibatches of size $20$, and validate the model over $20$ validation trajectories every $5$ training steps. Performance is measured by recording the cost of the decentralized controller, given by the sum of the deviations of each agent's velocity from the mean, relative to the cost of the centralized controller. The test set consists of $20$ trajectories with the same duration.

In decentralized control problems, transferability is important because, since we are training to imitate a centralized controller, the learning architecture has to be trained offline. Hence, the networks observed during training are different than those observed during execution, and typically smaller, because the cost of training graph filters and GNNs can be prohibitive for large graphs.
To assess whether the policies learned with graph filters and GNNs are transferable in the flocking problem, we perform the following experiment. We train the models on networks of size $n=25, 37, 50, 75, 87$. Then, we test them on both the original network and on a network with $n=100$ agents, and record the difference between the cost achieved on the original network and on the $100$-agent network relative to the cost on the original network. The results of this experiment for $10$ random realizations of the dataset are shown in Fig. \ref{fig:flocking1}.

In Fig. \ref{fig:flocking1}, we observe that the difference between the outputs of the graph filter and the GNN on these networks both decrease as the number of nodes increases. This is consistent with the asymptotic behavior of the transferability bounds in Thms. \ref{lemma:graphon-graph-filter} and \ref{thm:graphon-graph-nn}, which decrease with $n$. We also observe that, for fixed $n$, the relative cost difference of the GNN is smaller than the relative cost difference of the graph filter. This evidences that GNNs are more transferable than graph filters as discussed in Sec. \ref{sec:gnn-transferability}.

We further validate our results by comparing the 2-layer GNN with a GNN with same architecture but trained by penalizing filters with large Lipschitz constant. In this case, we minimize the sum of the MSE and of a penalization term given by the largest Lipschitz constant across all filters multiplied by a penalization factor. The results are shown in Fig. \ref{fig:flocking2}. Note that the models were only trained in networks with up to $n=62$ agents due to memory constraints. Although as expected the Lipschitz GNN is more transferable than the GNN, the difference is very small. This can be interpreted to mean that, in flocking, the optimal GNN filters have low variability regardless, which makes sense considering that flocking is a type of consensus task.

\subsection{Discussion}

The framework developed in this paper assumes the existence of an underlying random graph model in the form of a graphon. Two questions that arise are then whether the graphon is a suitable model in a given application and, if so, \textit{which} graphon. 
In the movie recommendation problem, the movie similarity graph is clearly sampled from a graphon. We do not know the closed-form expression of this graphon, but we do know that it is the limit of a sequence of correlation graphs. In practice, however, we do not need to know the graphon. The training graphs generated by sampling nodes uniformly at random from the large target graph are valid samples from this graphon.

In the decentralized robot control problem, the communication graphs are clearly \textit{not} sampled from a graphon, as their degree does not necessarily scale with the graph size. In this case, a better continuous model is perhaps the manifold, as we can define manifolds embedded in $d$-dimensional spaces and manifold graph models allow sampling sparser graphs. Still, even if in the flocking problem the graphon is not the most suitable limit object, we do observe a transferability behavior that agrees with our theoretical results. This suggests that transferability is a more global property of GNNs, not necessarily restricted to graphs sampled from graphons.

%% file: conclusions-transferability.tex


In this paper, we defined graphon convolutions and $\bbW$NNs, and showed that graph filters and GNNs sampled from them on stochastic and weighted graphs can be used to approximate graphon convolutions and $\bbW$NNs. Building upon this result, we then showed that graph filters and GNNs are transferable across weighted and stochastic graphs. The transferability error decreases with the size of the graphs, however, some spectral components are not transferable even when the graphs are large, which is a consequence of the fact that the graphon eigenvalues accumulate near zero. Furthermore, our transferability results reveal a tradeoff between the transferability and the spectral discriminability of graph convolutions, which is inherited by GNNs. In practice, however, in GNNs this tradeoff is alleviated by the pointwise nonlinearities. These findings were corroborated empirically in the problems of movie recommendation and decentralized robot control.

A limitation of our analysis is that graphons are only good models for limits of dense graphs. Therefore, a future research direction is to extend the transferability analysis of this paper to sparser graph limits, such as manifolds \cite{wang2021stability} and $L^p$ graphons \cite{borgs2019lp}. Another limitation of our work is that it does not address scenarios where the graph on which the GNN should be trained to meet a specific error allowance is still too large. To address this problem, we will leverage the transferability results in this paper to design training algorithms yielding more transferable GNNs \cite{cervino2021increase}. A third relevant research direction is the empirical evaluation of the tightness of the graph filter and GNN transferability error bounds. This is important for practical purposes, to inform the choice of the training graph size, as well as for theoretical purposes, to motivate attempts at tighter results.

%% file: app-transferability.tex



\section{Proof of Lemma 1} \label{sec:appendix0}

\begin{proof}[Proof of Lemma \ref{lemma:induced_filters}]
Denote the eigenvalues and eigenvectors of $\bbS_n$ $\lambda_i(\bbS_n)$ and $\bbv_i^n$, and the eigenvalues and eigenfunctions of $\bbW_n$ $\lambda_i(T_{\bbW_n})$ and $\varphi_i^n$. Consider the GFT $[\hbx_n]_i = \bbv_i^\Hr\bbx_n$ and the WFT $[\hat{X}_n]_i = \int_0^1 X_n(u) \varphi_i^n(u)du$.
The proof follows from the following lemma adapted from \cite[Lemma 2]{ruiz2019graphon}.
\begin{lemma}\label{T:induced_graphon}
Let $(\bbW_n,X_n)$ be the graphon signal induced by the  graph signal $(\bbG_n,\bbx_n)$. Then, for $i \in \ccalL \subseteq \mbZ \setminus \{0\}$, $|\ccalL| = n$, we have
\begin{align*}
	\lambda_i(T_{\bbW_n}) &= \frac{\lambda_{i}(\bbS_n)}{n}
	\\
	\varphi_i^n(u) &=
		\sum_{j=1}^n [\bbv^n_i]_{j} \sqrt{n} \mbI\left( u \in I_j \right)
	\\
	[\hat{X}_n]_i &= \dfrac{[\hbx_n]_i}{\sqrt{n}}
\end{align*}
For $i \notin \ccalL$, we let $\lambda_i(T_{\bbW_n}) = [\hat{X}_n]_i = 0$ and $\varphi_i^n = \psi_i$ such that $\{\varphi_i^n\} \cup \{\psi_i\}$ forms an orthonormal basis of $L^2([0,1])$.
\end{lemma}

 Explicitly, use \eqref{eqn:spec-graphon_filter} and the definition of the WFT to write the induced graphon filter as
\begin{equation}
    Y_n = \sum_{k=0}^{K-1} h_k \sum_{i \in \mbZ \setminus \{0\}} \lambda^k_i(T_{\bbW_n}) [\hat{X}_n]_i \varphi_i^n \text{.}
\end{equation}
Substituting $\lambda_i(T_{\bbW_n})$, $[\hat{X}_n]_i$ and $\varphi_i^n$ per Lemma \ref{T:induced_graphon}, we get
\begin{equation}
    Y_n(u) = \sum_{k=0}^{K-1} h_k \sum_{i \in \mbZ \setminus \{0\}}\frac{\lambda^k_i(\bbS_n)}{n^k} \frac{[\hbx_n]_i}{\sqrt{n}}\sum_{j=1}^n [\bbv^n_i]_{j} \sqrt{n} \mbI\left( u \in I_j \right)
\end{equation}
and using \eqref{eqn:spec-lsi-gf} and the definition of the iWFT,
\begin{equation}
    Y_n(u) = \sum_{j=1}^n\left[\sum_{k=0}^{K-1} h_k \frac{\bbS^k_n}{n^k} \bbx_n\right]_j\mbI\left( u \in I_j \right)
\end{equation}
which concludes the proof.
\end{proof}

\section{Proof of Props. 1 and 2} \label{sec:appendixA.2}

\begin{proof}[Proof of Prop. \ref{lemma:graphon-graph-filter-template}] 
Prop. \ref{lemma:graphon-graph-filter-template} is obtained by plugging the upper bounds on $\|\bbW-\mbW_n\|$ and $\|X_n-\mbX_n\|$ provided by Props. \ref{prop:w-w',template} and \ref{prop:x-x', template} below.

\begin{proposition} \label{prop:w-w',template}
Let $\bbW: [0,1]^2 \to [0,1]$ be an $A_w$-Lipschitz graphon, and let $\mbW_n := \bbW_{\mbG_n}$ be the graphon induced by the template graph $\mbG_n$ generated from $\bbW$ as in Definition \ref{def:template}. It holds that
\begin{equation*}
    \|\bbW-\mbW_n\| \leq \dfrac{2A_w}{n} \text{.}
\end{equation*}
\end{proposition}
\begin{proof}
Let $I_i = [(i-1)/n, i/n)$ for $1 \leq i \leq  n-1$ and $I_n = [(n-1)/n,1]$. 
Since the graphon is Lipschitz, for any $u \in I_i$, $v \in I_j$, $1 \leq i,j \leq n$, we have
\begin{align*}
|\bbW(u,v)-\mbW_n(u,v)| &\leq A_w\max\left(\left|u-\frac{i-1}{n}\right|,\left|\frac{i}{n}-u\right|\right) \\
&+ A_w\max\left(\left|v-\frac{j-1}{n}\right|,\left|\frac{j}{n}-v\right|\right) \\
&\leq \frac{A_w}{n} + \frac{A_w}{n} = \frac{2A_w}{n}\text{.}
\end{align*}
We can then write
\begin{align*}
\|\bbW-\mbW_n\|^2 &= \int_0^1 |\bbW(u,v)-\mbW_n(u,v)|^2 du dv \\
&\leq \int_0^1 \left(\frac{2A_w}{n}\right)^2 du dv = \left(\frac{2A_w}{n}\right)^2
\end{align*}
which concludes the proof.
\end{proof}

\begin{proposition} \label{prop:x-x', template}
Let $X \in L_2([0,1])$ be an $A_x$-Lipschitz graphon signal, and let $\mbX_n$ be the graphon signal induced by the graph signal $\mathbb{x}_n$ obtained from $X$ on the template graph $\mbG_n$ [cf. Definition \ref{def:template}], i.e., $[\bbx_n]_i = X((i-1)/n)$ for $1 \leq i \leq n$. It holds that
\begin{equation*}
    \|X-\mbX_n\| \leq \dfrac{{A_x}}{n} \text{.}
\end{equation*}
\end{proposition}
\begin{proof}
Let $I_i = [(i-1)/n, i/n)$ for $1 \leq i \leq  n-1$ and $I_n = [(n-1)/n,1]$. 
Since the graphon signal is Lipschitz, for any $u \in I_i$, $1 \leq i \leq n$, we have
\begin{align*}
|X(u) - \mbX_n(u)| \leq A_x\max\left(\left|u-\frac{i-1}{n}\right|,\left|\frac{i}{n}-u\right|\right) \leq \frac{A_x}{n}
\end{align*}
We can then write
\begin{align*}
\|X-\mbX_n\|^2 &= \int_0^1 |X(u) - \mbX_n(u)|^2 du \\
&\leq \int_0^1 \left(\frac{A_x}{n}\right)^2 du = \left(\frac{A_x}{n}\right)^2
\end{align*}
which completes the proof.
\end{proof}

\end{proof}

\begin{proof}[Proof of Prop. \ref{lemma:graphon-graph-filter-weighted}] 
Prop. \ref{lemma:graphon-graph-filter-weighted} is obtained from Thm. \ref{lemma:graphon-graph-filter} by upper bounding  $\|\bbW-\mbW_{\tilde{n}}\|$ and $\|X_n-\mbX_{\tilde{n}}\|$, i.e., the difference between the graphon and the graphon induced by the weighted graph $\mbG_{\tilde{n}}$, and the difference between the graphon signal and the graphon signal induced by $\mathbb{x}_{\tilde{n}}$. 
These bounds are obtained from Prop. \ref{prop:order_statistics} below.

\begin{proposition} \label{prop:order_statistics}
Fix $\chi_1, \chi_2 \in (0,0.3]$ and let $n \geq 4/\chi_2$. Let $U_1, U_2, \ldots, U_n$ be $n$ independently and uniformly distributed random variables on $[0,1]$ and let $U_{(1)}, U_{(2)}, \ldots, U_{(n)}$ denote their order statistics. Let $U_{(0)}=0$ and $U_{(n+1)}=1$ and define $S_i = U_{(i)}-U_{(i-1)}$ for $1 \leq i \leq n+1$. With probability at least $[1-2\chi_1]\times[1-\chi_2]$, the $(n+1)$th order statistic of $S_{i}$ satisfies
\begin{equation*}
    S_{(n+1)} \leq \dfrac{1}{n} \log{\left(\dfrac{{(n+1)}^2}{\log{(1-\chi_1)^{-1}}}\right)}\text{.}
\end{equation*}
\end{proposition}
\begin{proof}
For $0 \leq u_{(1)} \leq \ldots \leq u_{(n)} \leq 1$, the probability density of the order statistics $U_{(0)}, U_{(1)}, \ldots, U_{(n)}, U_{(n+1)}$ is given by
\begin{align}
\begin{split}
f_{U_{(0)}, U_{(1)},\ldots,U_{(n)},U_{(n+1)}}(0,u_{(1)},\ldots,u_{(n)},1) \\ 
= f_{U_{(1)},\ldots,U_{(n)}}(u_{(1)},\ldots,u_{(n)}) = n!\ \text{.}
\end{split}
\end{align}
Hence, we can also write
\begin{align}
\begin{split}
f_{S_{1},\ldots,S_{n+1}}(s_{1},\ldots,s_{n+1}) = n!
\end{split}
\end{align}
for $s_1 + \ldots + s_{n+1} = 1$. This probability density is the same as that of $n$ independent exponential random variables $X_i$ with parameter $\lambda=1$ divided by their sum. Letting $Y_i = X_i/\sum_{i=1}^n X_i$, we have
\begin{equation}
f_{Y_1,\ldots,Y_n}(y_1,\ldots,y_n) = n!
\end{equation}
for $0 \leq y_1 + \ldots + y_n \leq 1$. Thus, if we add a $(n+1)$th random variable $Y_{n+1}$ and set $y_1 + \ldots + y_{n+1} = 1$, the joint distribution of the $Y_i$ becomes
\begin{equation}
f_{Y_1,\ldots,Y_n,Y_{n+1}}(y_1,\ldots,y_{n+1}) = n!\ \text{.}
\end{equation}
This implies that the distributions of $S_1, \ldots S_{n+1}$ and $Y_1, \ldots, Y_{n+1}$ are equivalent. Observing that 
\begin{align}
\begin{split}
P\bigg(\max_{1 \leq i \leq n+1} S_i &= S_{(n+1)} \leq s\bigg) = P(S_1 \leq s, \ldots, S_{n+1} \leq s) \\
&= P(Y_1 \leq s, \ldots, Y_{n+1} \leq s) \\
&= P(Y_{(n+1)} \leq s)
\end{split}
\end{align}
we therefore conclude the equivalence of distribution 
\begin{equation}
S_{(n+1)} \equiv \frac{X_{(n+1)}}{\sum_{i=1}^{n+1}X_i}\ \text{.}
\end{equation}

Leveraging the equivalence above, we write the distribution of the largest spacing as
\begin{align}
\begin{split}
P&((n+1)S_{(n+1)}-\log{(n+1)} \leq x) \\
&= P\bigg(X_{(k+1)} \leq (x + \log{(n+1)})\frac{\sum_{i=1}^{n+1}X_{i}}{n+1}\bigg) \\
&=P(X_{(n+1)} - \log{(n+1)} \leq x + Z_{n+1})
\end{split}
\end{align}
where $Z_{n+1}$ is defined as
\begin{equation}
Z_{n+1} = (x+\log{(n+1)})\bigg(\frac{\sum_{i=1}^{n+1} X_i}{n+1} - 1\bigg) \ \text{.}
\end{equation}

Since $Z_{n+1}$ has expectation zero and variance $(x + \log{(n+1)})^2/(n+1)$, using the Chebyshev inequality we can write
\begin{align}
\begin{split}
P\bigg(|Z_{n+1}| \geq  \frac{x+\log{(n+1)}}{\sqrt{\chi_2(n+1)}}\bigg) \leq \chi_2\ \text{.}
\end{split}
\end{align}
Hence, conditioned on the event that $|Z_{n+1}| \leq  {(x+\log{(n+1)})}/{\sqrt{\chi_2(n+1)}}$, which we will denote $\ccalA$,
we have
\begin{align}\label{eqn:prob_approx_chi2}
\begin{split}
P((n+1)&S_{(n+1)}-\log{(n+1)} \leq x \ |\ \ccalA) \\
&= P(X_{(n+1)} - \log{(n+1)} \leq x + Z_{n+1}\ |\ \ccalA) \\
&\geq P(X_{(n+1)} - \log{(n+1)} \leq x - z)
\end{split}
\end{align}
for $n \geq 1/\chi_2$, where $z = ({x+\log{(n+1)}})/{\sqrt{\chi_2(n+1)}}$. 

The probability on the right hand side of \eqref{eqn:prob_approx_chi2} can be written explicitly as
\begin{align}
\begin{split}
P(X_{(n+1)}-\log{(n+1)} \leq x-z) = \bigg(1-\frac{e^{-(x-z)}}{n+1} \bigg)^{n+1}\ \text{.}
\end{split}
\end{align}
Leveraging the fact that $\lim_{n\to \infty}(1-a/n)^n = e^{-a}$, to be able to express $x$ in terms of this probability we will approximate it as
\begin{align} \label{eqn:exp_approx}
\begin{split}
P(X_{(n+1)}&-\log{(n+1)} \leq x-z) \\
&= \bigg(1-\frac{e^{-(x-z)}}{n+1} \bigg)^{n+1} \approx \exp{\big(-e^{-(x-z)}\big)}\ \text{.}
\end{split}
\end{align}
Thus, we need to quantify the error incurred in this approximation. 

Let $m=n+1$ and $a=e^{-(x-z)}$. Using this notation, $(1-a/m)^m$ can be expressed as $\exp{(m\log{(1-a/m)})}$. This allows us to compare exponents, which we will do by computing the Taylor series approximation of $m\log{(1-a/m)}$ around $a=0$. Subtracting $-a$ on both sides, we get
\begin{align}
\begin{split}
m\log{(1-a/m)} &- (-a) = \frac{a^2}{2m} - \frac{a^3}{3m^2} + \frac{a^4}{4m^3} - \ldots \\
&= \frac{a^2}{2m}\bigg(1-\frac{2a}{3m}+\frac{2a^2}{4m^2} - \ldots\bigg) \\
&\leq \frac{a^2}{2m} \mbox{ for } a \leq m\ \text{.}
\end{split}
\end{align}
Therefore, since $m\log{(1-a/m)}+a \geq 0$ for $a\leq m$, we can write
\begin{align}
\begin{split}
\bigg|\bigg(1-\frac{a}{m}\bigg)^m - e^{-a}\bigg| = \frac{|e^{m\log{(1-a/m)}+a}-1|}{e^a} \leq \frac{e^{a^2/2m}-1}{e^a}
\end{split}
\end{align}
by which we conclude
\begin{align} \label{eqn:prob_approx_chi1}
\begin{split}
P((n+1)S_{(n+1)}-\log{(n+1)}&\leq x\ |\ \ccalA) \\
&\geq e^{-a} - \frac{e^{a^2/2m}-1}{e^a} \text{.}
\end{split}
\end{align}

Now let $\chi_1 \in (0,0.3]$ and $1-\chi_1 = e^{-a} = \exp{-e^{-(x-z)}}$ (note that for these values of $\chi_1$ we always have $a \leq m$). This allows writing
\begin{align}
\begin{split}
x &= \frac{{\log{(n+1)}} - {\log{\log{(1-\chi_1)^{-1}}}\sqrt{\chi_2(n+1)}}}{\sqrt{\chi_2(n+1)}-1} \\
&\leq \frac{\log{(n+1)}}{\sqrt{\chi_2(n+1)}} - \log{\log{(1-\chi_1)^{-1}}}\ \text{.}
\end{split}
\end{align}
Hence, we have
\begin{align}
\begin{split}
\frac{x+\log{(n+1)}}{(n+1)} &\leq \frac{\log{(n+1)}}{n+1} \\
&+ \frac{\log{(n+1)}}{(n+1)(\sqrt{\chi_2(n+1)}-1)} \\
&- \frac{\log{\log{(1-\chi_1)^{-1}}}}{n+1}  \\
&\leq \dfrac{1}{n} \log{\left(\dfrac{{(n+1)}^2}{\log{(1-\chi_1)^{-1}}}\right)}\mbox{ for } n \geq \frac{4}{\chi_2}
\end{split}
\end{align}
which, together with \eqref{eqn:prob_approx_chi1}, implies 
that
\begin{align} \label{eqn:prob_approx_chi1.2}
\begin{split}
P\Bigg(S_{(n+1)} \leq \dfrac{1}{n} &\log{\bigg(\dfrac{{(n+1)}^2}{\log{(1-\chi_1)^{-1}}}\bigg)}\ \bigg|\ \ccalA\Bigg) \\
&\geq  e^{-a} - \frac{e^{a^2/2m}-1}{e^a} \\
&\geq (1-\chi_1)(1-(e^{a^2/2m}-1)) \\
&\geq (1-\chi_1)(2-e^{a^2/2m})
\end{split} 
\end{align}
for $n \geq 4/\chi_2$.

We conclude by quantifying the error introduced by $2-e^{a^2/2m}$.
Since $a = \log{(1-\chi_1)^{-1}}$ and $\chi_1 \leq 0.3$, $e^{a^2/2m}$ can be upper bounded as
\begin{align}
\begin{split}
e^{a^2/2} &= (1-\chi_1)^{-1(\log{(1-\chi_1)^{-1}})/2m} \\
&= \frac{1}{(1-\chi_1)^{-\log{(1-\chi_1)}/2m}} \\
&\leq \frac{1}{1-\chi_1}
\end{split}
\end{align}
where the last inequality follows from the fact that $-0.36 \leq \log{(1-\chi_1)}\leq 0$. Finally, plugging this inequality into \eqref{eqn:prob_approx_chi1.2} we obtain that
\begin{align}
\begin{split}
P\Bigg(S_{(n+1)} \leq \dfrac{1}{n} \log{\bigg(\dfrac{{(n+1)}^2}{\log{(1-\chi_1)^{-1}}}\bigg)}\ \bigg|\ \ccalA\Bigg) \geq 1-2\chi_1 \text{.}
\end{split} 
\end{align}
The result of Prop. \ref{prop:order_statistics} then follows by observing that, for an event $\ccalB$, $P(\ccalB) \geq P(\ccalB|\ccalA)P(\ccalA)$. Hence,
\begin{align}
\begin{split}
P\Bigg(S_{(n+1)} \leq \dfrac{1}{n} \log{\bigg(\dfrac{{(n+1)}^2}{\log{(1-\chi_1)^{-1}}}\bigg)}\Bigg) \\
\geq (1-2\chi_1)(1-\chi_2)
\end{split}
\end{align}
for $n \geq 4/\chi_2$ and $\chi_1,\chi_2 \in (0,0.3]$.
\end{proof}

Using Prop. \ref{prop:order_statistics}, we can now derive an upper bound for $\|\bbW-\mbW_{\tilde{n}}\|$ and $\|X_n-\mbX_{\tilde{n}}\|$ in Props. \ref{prop:w-w',weighted} and \ref{prop:x-x',weighted}. 

\begin{proposition} \label{prop:w-w',weighted}
Let $\bbW: [0,1]^2 \to [0,1]$ be an $A_w$-Lipschitz graphon, and let $\mbW_{\tilde{n}} := \bbW_{\mbG_{\tilde{n}}}$ be the graphon induced by the template graph $\mbG_{\tilde{n}}$ generated from $\bbW$ as in Definition \ref{def:weighted}. Fix $\chi_1, \chi_2 \in (0,0.3]$ and let $n \geq 4/\chi_2$. With probability at least $[1-2\chi_1]\times[1-\chi_2]$, it holds that
\begin{equation*}
    \|\bbW-\mbW_{\tilde{n}}\| \leq \dfrac{2A_w}{n} \log{\left(\dfrac{{(n+1)}^2}{\log{(1-\chi_1)^{-1}}}\right)}\text{.}
\end{equation*}
\end{proposition}
\begin{proof}
Let $U_1, \ldots, U_n$ be $n$ independently and uniformly distributed random variables on $[0,1]$ and let $U_{(1)}, \ldots, U_{(n)}$ denote their order statistics. Setting $U_{(0)}=0$ and $U_{(n+1)}=1$, we can write the spacings between the $U_{(i)}$ as $S_i = U_{(i)}-U_{(i-1)}$ for $1 \leq i \leq n+1$ and their order statistics as $S_{(i)}$.
Since the graphon is Lipschitz, for any $u, v \in [0,1]$, we have
\begin{align*}
|\bbW(u,v)-\mbW_{\tilde{n}}(u,v)| &\leq A_w\max_i S_i + A_w\max_i S_i \\
&= 2 A_w S_{(n+1)} \\
&\leq \dfrac{2A_w}{n} \log{\left(\dfrac{{(n+1)}^2}{\log{(1-\chi_1)^{-1}}}\right)}
\end{align*}
with probability at least $[1-2\chi_1]\times[1-\chi_2]$ for $n \geq 4/\chi_2$, where the last inequality follows from Prop. \ref{prop:order_statistics}.
We can then write
\begin{align*}
\|\bbW-\mbW_{\tilde{n}}\|^2 &= \int_0^1 |\bbW(u,v)-\mbW_{\tilde{n}}(u,v)|^2 du dv \\
&\leq \int_0^1 \left( \dfrac{2A_w}{n} \log{\left(\dfrac{{(n+1)}^2}{\log{(1-\chi_1)^{-1}}}\right)}\right)^2 du dv \\
&= \left(\dfrac{2A_w}{n} \log{\left(\dfrac{{(n+1)}^2}{\log{(1-\chi_1)^{-1}}}\right)}\right)^2
\end{align*}
which concludes the proof.
\end{proof}

\begin{proposition} \label{prop:x-x',weighted}
Let $X \in L_2([0,1])$ be an $A_x$-Lipschitz graphon signal, and let $\mbX_{\tilde{n}}$ be the graphon signal induced by the graph signal $\mathbb{x}_{\tilde{n}}$ obtained from $X$ on the template graph $\mbG_{\tilde{n}}$ [cf. Definition \ref{def:weighted}], i.e., $[\bbx_{\tilde{n}}]_i = X((i-1)/n)$ for $1 \leq i \leq n$. Fix $\chi_1, \chi_2 \in (0,0.3]$ and let $n \geq 4/\chi_2$. With probability at least $[1-2\chi_1]\times[1-\chi_2]$, it holds that
\begin{equation*}
    \|X-\mbX_{\tilde{n}}\| \leq \dfrac{{A_x}}{n}\log{\left(\dfrac{{(n+1)}^2}{\log{(1-\chi_1)^{-1}}}\right)}\text{.}
\end{equation*}
\end{proposition}
\begin{proof}
Let $U_1, \ldots, U_n$ be $n$ independently and uniformly distributed random variables on $[0,1]$ and let $U_{(1)}, \ldots, U_{(n)}$ denote their order statistics. Setting $U_{(0)}=0$ and $U_{(n+1)}=1$, we can write the spacings between the $U_{(i)}$ as $S_i = U_{(i)}-U_{(i-1)}$ for $1 \leq i \leq n+1$ and their order statistics as $S_{(i)}$.
Since the graphon is Lipschitz, for any $u, v \in [0,1]$, we have
\begin{align*}
|X(u) - \mbX_{\tilde{n}}(u)| &\leq A_x\max_i S_i = A_x S_{(n+1)} \\
&\leq \dfrac{A_x}{n} \log{\left(\dfrac{{(n+1)}^2}{\log{(1-\chi_1)^{-1}}}\right)}
\end{align*}
with probability at least $[1-2\chi_1]\times[1-\chi_2]$ for $n \geq 4/\chi_2$, where the last inequality follows from Prop. \ref{prop:order_statistics}.
We can then write
\begin{align*}
\|X-\mbX_n\|^2 &= \int_0^1 |X(u) - \mbX_{\tilde{n}}(u)|^2 du \\
&\leq \int_0^1 \left(\frac{A_x}{n}\log{\left(\dfrac{{(n+1)}^2}{\log{(1-\chi_1)^{-1}}}\right)}\right)^2 du \\
&= \left(\frac{A_x}{n}\log{\left(\dfrac{{(n+1)}^2}{\log{(1-\chi_1)^{-1}}}\right)}\right)^2
\end{align*}
which completes the proof.
\end{proof}

\end{proof}

\section{Proof of Lemma 1} \label{sec:appendixB}

Following the same reasoning used in the proof of Thm. \ref{lemma:graphon-graph-filter}, we fix a constant $0 < c < 1$ and decompose the filter $h$ into filters $h^{\geq c}$ and $h^{< c}$ defined as
\begin{align}
&h^{\geq c}(\lambda) \begin{cases} 0 &\mbox{if } |\lambda| < c \\
h(\lambda)-h(c) & \mbox{if } \lambda \geq c \\
h(\lambda)-h(-c) & \mbox{if } \lambda \leq -c \end{cases} \quad \mbox{and}\\
&h^{<c}(\lambda) \begin{cases} h(\lambda) &\mbox{if } |\lambda| < c \\
h(c) & \mbox{if } \lambda \geq c \\
h(-c) & \mbox{if } \lambda \leq -c \end{cases}
\end{align}
so that $h= h^{\geq c} + h^{<c}$. This time, we start by analyzing the transferability of the filter $\bbH^{<c}(\bbS) = \bbV h^{<c}(\bbLam/n)\bbV^\Hr$.
Note that, because the function $h^{<c}$ only varies in the interval $(-c,c)$ and has Lipschitz constant $a_h$, we can bound $\|\bbH^{<c}(\mbS_{\tilde{n}})\mathbb{x}_{\tilde{n}}-\bbH^{<c}(\bbS_{\tilde{n}})\bbx_{\tilde{n}}\|$ as
\begin{align} \label{eqn_lemma1_3}
\begin{split}
\|\bbH^{<c}(\mbS_{\tilde{n}})&\mathbb{x}_{\tilde{n}}-\bbH^{<c}(\bbS_{\tilde{n}})\bbx_{\tilde{n}}\| \\
&= \|\bbH^{<c}(\mbS_{\tilde{n}})\bbx_{\tilde{n}}-\bbH^{<c}(\bbS_{\tilde{n}})\bbx_{\tilde{n}}\|
\\
&\leq \|(h(0)+a_h c)
-(h(0)-a_h c)\|\|\bbx_{\tilde{n}}\| \\
&= 2a_h c\|\bbx_{\tilde{n}}\|
\end{split}
\end{align}
since $\mathbb{x}_{\tilde{n}} = \bbx_{\tilde{n}}$. 

Next, we analyze the transferability of the filter $\bbH^{\geq c}(\bbS) = \bbV h^{\geq c}(\bbLam/n)\bbV^\Hr$.
Using the triangle inequality, we bound $\|\bbH^{\geq c}(\mbS_{\tilde{n}})\bbx_{\tilde{n}}-\bbH^{\geq c}(\bbS_{\tilde{n}})\bbx_{\tilde{n}}\|$ as
\begin{align*}
\begin{split}
\|&\bbH^{\geq c}(\mbS_{\tilde{n}})\bbx_{\tilde{n}}-\bbH(\bbS_{\tilde{n}})\bbx_{\tilde{n}}\|  \\
&\leq \left\|\sum_i \left(h^{\geq c}(\lambda_i(\mbS_{\tilde{n}})/n)-h^{\geq c}(\lambda_i(\bbS_{\tilde{n}})/n)\right)\mathbb{v}_i^\Hr\bbx_{\tilde{n}} \mathbb{v}_i\right\|\mbox{  {\bf (1)}}\\
&+\left\|\sum_i h^{\geq c}(\lambda_i(\bbS_{\tilde{n}})/n)\left(\mathbb{v}_i^\Hr\bbx_{\tilde{n}} \mathbb{v}_i - \bbv_i^\Hr\bbx_{\tilde{n}} \bbv_i \right)\right\|\text{.} \quad \ \ \ \ \ \ \mbox{\bf (2)} 
\end{split}
\end{align*}

Using the Lipschitz property of the filter together with Prop. \ref{prop:eigenvalue_diff} and the Cauchy-Schwarz inequality, we can rewrite {\bf (1)} as
\begin{align*}
\begin{split}
\left\|\sum_i \left(h^{\geq c}(\lambda_i(\mbS_{\tilde{n}})/n)-h^{\geq c}(\lambda_i(\bbS_{\tilde{n}})/n)\right)\mathbb{v}_i^\Hr\bbx_{\tilde{n}} \mathbb{v}_i\right\| \\
\leq A_h \dfrac{\|{\mbS}_{\tilde{n}}-\bbS_{\tilde{n}}\|}{n} \|\bbx_{\tilde{n}}\|\ .
\end{split}
\end{align*}
To bound $\|\mbS_{\tilde{n}}-\bbS_{\tilde{n}}\|$, we will need the following proposition adapted from \cite[Thm. 1]{chung2011spectra}.

\begin{proposition}
Let $\mbG_{\tilde{n}}$ be a template graph with GSO $\mbS_{\tilde{n}}$ and let $\bbG_{\tilde{n}}$ be a stochastic graph with GSO $\bbS_{\tilde{n}}$ such that $[\bbS_{\tilde{n}}]_{ij} = \mbox{Bernoulli}([\mbS_{\tilde{n}}]_{ij})$, i.e., $\mbS_{\tilde{n}} = \mbE(\bbS_{\tilde{n}})$. Let $d_{\bbG_{\tilde{n}}}$ denote the maximum expected degree of $\bbG_{\tilde{n}}$. Let $\chi > 0$, and suppose that for $n$ sufficiently large, $d_{\bbG_{\tilde{n}}} > 4\log{(2n/\chi)}/9$. Then with probability at least $1-\chi$, for $n$ sufficiently large,
\begin{equation*}
\|\bbS_{\tilde{n}}-\mbS_{\tilde{n}}\| \leq \sqrt{4d_{\bbG_{\tilde{n}}}\log{(2n/\chi)}} \leq \sqrt{4n\log{(2n/\chi)}} \text{.}
\end{equation*}
\end{proposition}
\begin{proof}
See \cite[Thm. 1]{chung2011spectra}.
\end{proof}

Hence, the term $\|\mbS_{\tilde{n}}-\bbS_{\tilde{n}}\|$ can be bounded with probability $1-\chi$ provided that $d_{\bbG_{\tilde{n}}}$, the maximum expected degree of the graph $\bbG_n$, satisfies $d_{\bbG_n} > 4 \log (2n/\chi)/9$. This is the case for graphs $\bbG_{\tilde{n}}$ for which $n$ satisfies assumption AS\ref{as4}.
To see this, write
\begin{align*}
\frac{d_{\bbG_{\tilde{n}}}}{n} &= \frac{1}{n} \max_i \left(\sum_{j=1}^n[\mbS_{\tilde{n}}]_{ij}\right) = \frac{1}{n}\max_i \left(\sum_{j=1}^n \bbW(u_i,u_j)\right) \\
&= \max_{u \in [0,1]} \int_0^1 \bbW(u,v) dv \\
&\geq \max_{u \in [0,1]} \left(\int_0^1 \bbW(u,v) dv - \int_0^1 |\bbD(u,v)| dv\right)
\end{align*}
where $\bbD(u,v)$ is the degree function of $\bbW$, i.e., $\bbD(u,v)= \int_0^1 \bbW(u,v) dv$. Using the inverse triangle inequality for the maximum, we get
\begin{align*}
\frac{d_{\bbG_{\tilde{n}}}}{n} &\geq \max_{u \in [0,1]} \int_0^1 \bbW(u,v) dv - \max_{u \in [0,1]}\int_0^1 |\bbD(u,v)| dv \\
&= d_\bbW - \max_{u \in [0,1]}\int_0^1 |\bbD(u,v)| dv\ .
\end{align*}
Hence, it suffices to find an upper bound for the maximum of the integral of $|\bbD(u,v)|$. Let $F(u) = 1$. Then,
\begin{align*}
\max_u \int_0^1 |\bbD&(u,v)|dv = \max_u \int_0^1 |\bbD(u,v)F(v)|dv \\
&\leq  \sqrt{\int_0^1 \left(\int_0^1 |\bbD(u,v)F(v)|dv\right)^2 du} \\
&\leq \sqrt{\int_0^1 \int_0^1 |\bbD(u,v)|^2dv\int_0^1|F(v)|^2dvdu} \\
&= \sqrt{\int_0^1 \int_0^1 |\bbD(u,v)|^2dvdu}
\end{align*}
where the first inequality follows from the fact that the $L^2$ norm dominates the $L^\infty$ norm, and the second from Cauchy-Schwarz. Given the graphon's Lipschitz property, we know that $|\bbD(u,v)|$ is at most $2A_w/n$.  Hence,
\begin{align*}
\max_u \int_0^1 |\bbD(u,v)|dv \leq \sqrt{\int_0^1 \int_0^1 \frac{4A_w^2}{n^2}dvdu} = \frac{2A_w}{n}
\end{align*}
and, thus, ${d_{\bbG_{\tilde{n}}}} \geq n d_\bbW - {2A_w}$, which by assumption AS\ref{as4} entails $d_{\bbG_{\tilde{n}}} > 4 \log (2n/\chi)/9$.

Given that assumption AS\ref{as4} is satisfied, with probability at least $1-\chi_3$ we can write
\begin{align} \label{eqn_lemma1_1}
\begin{split}
\Big\|\sum_i \big(h^{\geq c}(\lambda_i(\mbS_{\tilde{n}})/n)&-h^{\geq c}(\lambda_i(\bbS_{\tilde{n}})/n)\big)\mathbb{v}_i^\Hr\bbx_{\tilde{n}} \mathbb{v}_i\Big\| \\
&\leq A_h \dfrac{2\sqrt{n\log(2n/\chi_3)}}{n} \|\bbx_{\tilde{n}}\|\ \text{.}
\end{split}
\end{align}

For \textbf{(2)}, we use the triangle and Cauchy-Schwarz inequalities to decompose the norm difference as
\begin{align*}
\begin{split}
\left\|\sum_i h^{\geq c}(\lambda_i(\bbS_{\tilde{n}})/n)\left(\mathbb{v}_i^\Hr\bbx_{\tilde{n}} \mathbb{v}_i - \bbv_i^\Hr\bbx_{\tilde{n}} \bbv_i \right)\right\|
\\
\leq 2\sum_{i} \|h^{\geq c}(\lambda_i(\bbS_{\tilde{n}})/n)\| \|\bbx_{\tilde{n}}\| \|\mathbb{v}_i - {\bbv}_i\| \text{.}
\end{split}
\end{align*}
Prop. \ref{thm:davis_kahan} gives an upper bound to the norm difference between the eigenvectors, by which we obtain
\begin{align*}
\begin{split}
\left\|\sum_i h^{\geq c}(\lambda_i(\bbS_{\tilde{n}})/n)\left(\mathbb{v}_i^\Hr\bbx_{\tilde{n}} \mathbb{v}_i - \bbv_i^\Hr\bbx_{\tilde{n}} \bbv_i \right)\right\| \\
\leq \|\bbx_{\tilde{n}}\| \sum_{i} \|h^{\geq c}(\lambda_i(\bbS_{\tilde{n}})/n)\| \frac{\pi\|\mbS_{\tilde{n}}-\bbS_{\tilde{n}}\|}{d_i}
\end{split}
\end{align*}
where $d_i$ is the minimum between $\min(|\lambda_i(\mbS_{\tilde{n}}) - \lambda_{i+1}(\bbS_{\tilde{n}})|,|\lambda_i(\mbS_{\tilde{n}})-\lambda_{i-1}(\bbS_{\tilde{n}})|)$ and $\min(|\lambda_i(\bbS_{\tilde{n}}) - \lambda_{i+1}(\mbS_{\tilde{n}})|,|\lambda_i(\bbS_{\tilde{n}})-\lambda_{i-1}(\mbS_{\tilde{n}})|)$ for each $i$. Since the eigenvalues of the graphon induced by a graph are equal to the eigenvalues of the graph divided by $n$ \cite[Lemma 1]{ruiz2020graphonsp}, we have $n \delta^c_{\mbW_{\tilde{n}}\bbW_{\tilde{n}}} \leq d_i$ for all $i$. Leveraging this together with \cite[Thm. 1]{chung2011spectra}, it thus holds with probability $1-\chi_3$ that 
\begin{align} \label{eqn_lemma1_2}
\begin{split}
\left\|\sum_i h^{\geq c}(\lambda_i(\bbS_{\tilde{n}})/n)\left(\mathbb{v}_i^\Hr\bbx_{\tilde{n}} \mathbb{v}_i - \bbv_i^\Hr\bbx_{\tilde{n}} \bbv_i \right)\right\| \\ 
\leq \|\bbx_{\tilde{n}}\| B^c_{\bbW_{\tilde{n}}} \frac{2\pi\sqrt{n\log (2n/\chi_3)}}{n\delta^c_{\mbW_{\tilde{n}}\bbW_{\tilde{n}}}}\text{.}
\end{split}
\end{align}
Putting together \eqref{eqn_lemma1_3}, \eqref{eqn_lemma1_1} and \eqref{eqn_lemma1_2} completes the proof. \qed

\section{Proof of Thm. 2} \label{sec:appendixC}

To compute a bound for $\|Y-Y_{{n}}\|$, we start by writing it in terms of the last layer's features as
\begin{equation} \label{eqn:proof2.0}
\left\|Y-Y_{{n}}\right\|^2 = \sum_{f=1}^{F_L} \left\|X^f_{L} - X^f_{{n},L}\right\|^2.
\end{equation}
At  layer $\ell$ of the WNN $\Phi(X; \ccalH, \bbW)$, we have
\begin{align*}
\begin{split}
X^f_{\ell} = \sigma\left(\sum_{g=1}^{F_{\ell-1}}T_{\bbH_\ell^{fg}} X_{\ell-1}^g\right)
\end{split}
\end{align*}
and similarly for $\bbPhi(X_{{n}};\ccalH, \bbW_{{n}})$,
\begin{align*}
\begin{split}
X^f_{{n},\ell} = \rho\left(\sum_{g=1}^{F_{\ell-1}}T_{\bbH_{{n},\ell}^{fg}} X_{{n},\ell-1}^g\right)
\end{split}
\end{align*}
where $T_{\bbH_{{n},\ell}^{fg}}$ is the convolutional filter \eqref{eqn:lsi-wf} mapping feature $X_{\ell-1,{n}}^{g}$ to feature $X_{\ell,{n}}^{f}$. 
We can thus write $\|X^f_{\ell}-X^f_{{n},\ell}\|$ as
\begin{align*}
\begin{split}
\big\|X^f_{\ell}&-X^f_{{n},\ell}\big\| \\
&= \left\|\sigma\left(\sum_{g=1}^{F_{\ell-1}}T_{\bbH_\ell^{fg}} X_{\ell-1}^g\right)-\rho\left(\sum_{g=1}^{F_{\ell-1}}T_{\bbH_{{n},\ell}^{fg}} X_{{n},\ell-1}^g\right)\right\|
\end{split}
\end{align*}
and, since by assumption AS\ref{as5} $\sigma$ is normalized Lipschitz,
\begin{align*}
\begin{split}
\left\|X^f_{\ell}-X^f_{{n},\ell}\right\| &\leq \left\|\sum_{g=1}^{F_{\ell-1}}T_{\bbH_\ell^{fg}} X_{\ell-1}^g-T_{\bbH_{{n},\ell}^{fg}} X_{{n},\ell-1}^g\right\|\\
&\leq \sum_{g=1}^{F_{\ell-1}}\left\|T_{\bbH_\ell^{fg}} X_{\ell-1}^g-T_{\bbH_{{n},\ell}^{fg}} X_{{n},\ell-1}^g\right\|.
\end{split}
\end{align*}
where the second inequality follows from the triangle inequality. Looking at each feature $g$ independently, we apply the triangle inequality once again to obtain
\begin{align*}
\begin{split}
\left\|T_{\bbH_\ell^{fg}} X_{\ell-1}^g-T_{\bbH_{{n},\ell}^{fg}} X_{{n},\ell-1}^g\right\| \leq \left\|T_{\bbH_\ell^{fg}} X^g_{\ell-1}-T_{\bbH_{{n},\ell}^{fg}} X^g_{\ell-1}\right\| \\
+ \left\|T_{\bbH^{fg}_{{n},\ell}}\left(X^g_{\ell-1}-X^g_{{n},\ell-1}\right)\right\|.
\end{split}
\end{align*}
From Thm. \ref{lemma:graphon-graph-filter}, 
the first term on the RHS of this inequality is bounded by 
\begin{align} \label{eqn:step_change}
\begin{split}
&\big\|T_{\bbH_\ell^{fg}} X^g_{\ell-1}-T_{\bbH_{{n},\ell}^{fg}} X^g_{\ell-1}\big\| \\
&\leq \bigg(A_h + \frac{\pi B_{\bbW_{{n}}}^c}{\delta^c_{\bbW\bbW_{{n}}}}\bigg)
\|\bbW-\bbW_n\|\|X_{\ell-1}^g\| + 4A_h c\|X_{\ell-1}^g\|
\end{split}
\end{align}
Note that the term 
$(A_h c + 2)\|X-X_n\|$, which measures the distance between the input signals of the two WNNs, disappears here because the input is $X_{\ell-1}^g$ for both.
The second term can be decomposed using Cauchy-Schwarz and recalling that $|h(\lambda)| < 1$ for all graphon convolutions in the WNN (assumption AS\ref{as2}). We thus obtain a recursion for $\|X^f_{\ell}-X^f_{{n},\ell}\|$, which is given by
\begin{align} \label{eqn:proof2.1}
\begin{split}
\big\|X^f_{\ell}-&X^f_{n,\ell}\big\| \\
&\leq \bigg(A_h + \frac{\pi B_{\bbW_{{n}}}^c}{\delta^c_{\bbW\bbW_{{n}}}}\bigg)
\|\bbW-\bbW_n\|\|X_{\ell-1}^g\| \\
&+ 4A_h c\|X_{\ell-1}^g\| + \sum_{g=1}^{F_{\ell-1}}\left\|X^g_{\ell-1}-X^g_{{n},\ell-1}\right\| \text{.}
\end{split}
\end{align}

To solve the recursion in \eqref{eqn:proof2.1}, we need to compute the norm $\|X_{\ell-1}^g\|$. Since the nonlinearity $\sigma$ is normalized Lipschitz and $\sigma(0)=0$ by assumption AS\ref{as5}, this bound can be written as
\begin{align*}
\begin{split}
\left\|X_{\ell-1}^g\right\| \leq \left\|\sum_{g=1}^{F_{\ell-1}} T_{\bbH_\ell^{fg}} X_{\ell-1}^g \right\|
\end{split}
\end{align*}
and using the triangle and Cauchy Schwarz inequalities,
\begin{align*}
\begin{split}
\left\|X_{\ell-1}^g\right\| \leq \sum_{g=1}^{F_{\ell-1}}\vertiii{T_{\bbH_\ell^{fg}}} \left\|X_{\ell-1}^g \right\| \leq \sum_{g=1}^{F_{\ell-1}}\left\|X_{\ell-1}^g \right\|
\end{split}
\end{align*}
where the second inequality follows from $|h(\lambda)| < 1$. Expanding this expression with initial condition $X_0^g = X^g$ yields
\begin{align} \label{eqn:proof2.2}
\begin{split}
\left\|X_{\ell-1}^g\right\| \leq \prod_{\ell'=1}^{\ell-1} F_{\ell'} \sum_{g=1}^{F_{0}}\left\|X^g \right\|
\end{split}
\end{align}
and substituting it back in \eqref{eqn:proof2.1} to solve the recursion, we get
\begin{align} \label{eqn:proof2.3}
\begin{split}
\big\|X^f_{\ell}&-X^f_{{n},\ell}\big\| \leq L 
\bigg(A_h + \frac{\pi B_{\bbW_{{n}}}^c}{\delta^c_{\bbW\bbW_{{n}}}}\bigg)
\|\bbW-\bbW_n\|\\
&\times 
\left(\prod_{\ell'=1}^{\ell-1} F_{\ell'}\right) \sum_{g=1}^{F_{0}}\left\|X^g \right\| + 4LA_h c \left(\prod_{\ell'=1}^{\ell-1} F_{\ell'}\right) \sum_{g=1}^{F_{0}}\left\|X^g \right\| \\
&+ F_0
(A_h c + 2)\|X_g-X_{n}^g\|\text{.}
\end{split}
\end{align}

To arrive at the result of Thm. \ref{thm:graphon-graph-nn}, we  evaluate \eqref{eqn:proof2.3} with $\ell=L$ and substitute it into \eqref{eqn:proof2.0} to obtain
\begin{align} \label{eqn:proof2.4}
\begin{split}
\big\|Y-&Y_{{n}}\big\|^2 = \sum_{f=1}^{F_L} \left\|X^f_{L} - X^f_{{n},L}\right\|^2 \\
&\leq \sum_{f=1}^{F_L} \Bigg(L \bigg(A_h + \frac{\pi B_{\bbW_{{n}}}^c}{\delta^c_{\bbW\bbW_{{n}}}}\bigg)
\|\bbW-\bbW_n\|\\
&\times 
\left(\prod_{\ell'=1}^{\ell-1} F_{\ell'}\right) \sum_{g=1}^{F_{0}}\left\|X^g \right\| + 4LA_h c \left(\prod_{\ell'=1}^{\ell-1} F_{\ell'}\right) \sum_{g=1}^{F_{0}}\left\|X^g \right\| \\
&+ F_0
(A_h c + 2)\|X^g-X^g_n\|
\Bigg)^2.
\end{split}
\end{align}
Substituting $F_0 = F_L = 1$ and $F_\ell = F$ for $1 \leq \ell \leq L-1$ concludes the proof. \qed